%% file: emnlp2022.tex
\documentclass[11pt]{article}
\usepackage[]{EMNLP2022}
\usepackage{times}
\usepackage{latexsym}
\usepackage[T1]{fontenc}
\usepackage[utf8]{inputenc}
\usepackage{microtype}
\usepackage{inconsolata}
\usepackage{multirow}
\usepackage{latexsym}
\usepackage{multirow}
\usepackage{graphicx}
\usepackage{makecell}
\usepackage{amssymb}
\usepackage{amsmath}
\usepackage{xcolor}
\usepackage{array}
\usepackage{blindtext}
\usepackage{hyperref}
\usepackage{wrapfig}
\usepackage{caption}
\usepackage{subcaption}
\usepackage{booktabs}
\usepackage{colortbl}
\usepackage{amsthm}
\usepackage{algorithm}
\usepackage{algorithmic}
\usepackage{stackengine}

\newtheorem{assumption}{Assumption}
\newtheorem{theorem}{Theorem}
\newtheorem{lemma}{Lemma}

\title{Improving Sharpness-Aware Minimization with Fisher Mask\\ for Better Generalization on Language Models}

\newcommand{\co}{$^{\dagger}$}
\newcommand{\sig}{$^{\ddagger}$}
\newcommand{\gap}{\textsc{\textit{Fi}}\xspace}

\author{%
  Qihuang~Zhong$^{1}$\thanks{~~Work was done when Qihuang was interning at JD Explore Academy.},
  Liang~Ding$^{2}$,
  Li~Shen$^{2}$,
  Peng~Mi$^{3}$, 
  \textbf{Juhua~Liu}$^{4}$\thanks{~~Corresponding Authors: Juhua Liu (e-mail: liujuhua@whu.edu.cn), Bo Du (e-mail: dubo@whu.edu.cn)},
  \textbf{Bo~Du}$^{1}$\co,
  \textbf{Dacheng~Tao}$^{2}$ \\
  \fontsize{9.0pt}{\baselineskip}\selectfont $^{1}$ National Engineering Research Center for Multimedia Software, Institute of Artificial Intelligence, School of Computer Science \\ 
  \fontsize{9.0pt}{\baselineskip}\selectfont  and Hubei Key Laboratory of Multimedia and Network Communication Engineering, Wuhan University, China \\
  \fontsize{9.0pt}{\baselineskip}\selectfont $^{2}$ JD~Explore~Academy, China \quad $^{3}$ School of Informatics, Xiamen University, China \\
   \fontsize{9.0pt}{\baselineskip}\selectfont $^{4}$ Research Center for Graphic Communication, Printing and Packaging, and Institute of Artificial Intelligence, Wuhan University, China   \\
   \fontsize{9.0pt}{\baselineskip}\selectfont \texttt{\{zhongqihuang,liujuhua,dubo\}@whu.edu.cn}, \texttt{\{dingliang1,shenli100\}@jd.com}, \texttt{mipeng@stu.xmu.edu.cn}, \texttt{dacheng.tao@gmail.com}
}

\begin{document}
\maketitle

\input{section/abstract}
\input{section/introduction}
\input{section/related_work}
\input{section/method}

\input{tables/main_results}

\input{tables/main_results2}
\input{tables/smaller_model}

\input{section/experimental_setup}
\input{section/main_results}
\input{section/analysis}
\input{section/conclusion}

\input{section/limitations}
\input{section/acknowledgement.tex}
\bibliography{emnlp2022}
\bibliographystyle{acl_natbib}

\input{section/appendix}

\end{document}

%% file: section/abstract.tex
\begin{abstract}
Fine-tuning large pretrained language models on a limited training corpus usually suffers from poor generalization.
Prior works show that the recently-proposed sharpness-aware minimization (SAM) optimization method can improve the model generalization. However, SAM adds a perturbation to each model parameter equally (\textit{but not all parameters contribute equally to the optimization of training}), which we argue is sub-optimal and will lead to excessive computation. In this paper, we propose a novel optimization procedure, namely FSAM\footnote{~\url{https://github.com/WHU-ZQH/FSAM4PLM}}, which introduces a Fisher mask to improve the efficiency and performance of SAM. In short, instead of adding perturbation to all parameters, FSAM uses the Fisher information to identity the important parameters and formulates a Fisher mask to obtain the sparse perturbation, \textit{i.e., making the optimizer focus on these important parameters}.
Experiments on various tasks in GLUE and SuperGLUE benchmarks show that FSAM consistently outperforms the vanilla SAM by 0.67$\sim$1.98 average score among four different pretrained models. We also empirically show that FSAM works well in other complex scenarios, \textit{e.g.}, fine-tuning on generation tasks or limited training data.
 Encouragingly, when training data is limited, FSAM improves the SAM by a large margin, \textit{i.e.}, up to 15.1.
\end{abstract}

%% file: section/introduction.tex
\section{Introduction}
\label{sec_intro}

The ``pretraining-finetuning'' paradigm has become the \textit{de facto} standard for the community of natural language processing (NLP)~\cite{devlin2019bert,liu2019roberta,clark2019electra,raffel2020exploring,brown2020language,lewis2020bart}. Given a pretrained language model (PLM), the dominant fine-tuning manner is tuning the entire pretrained parameters for each downstream task~\cite{radford2018improving, devlin2019bert}. 
While fine-tuning the entire PLM can improve performance on a wide range of NLP tasks, it usually suffers from \textit{over-fitting} and \textit{poorer generalization ability}~\cite{xu2021raise, bahri-etal-2022-sharpness}, especially in the large-scale PLMs and limited training data scenarios. 

Hence, some existing efforts attempt to provide more regularization in the fine-tuning stage~\cite{zhang2018mixup,muller2019does,xu2021raise}, among which the optimization of the training loss is an intuitive and effective method. 
Specifically, motivated by the finding~\cite{keskar2016large, neyshabur2017exploring} that the smoother loss landscape refers to the better model generalization, \citet{foret2020sharpness} propose the ``\textbf{sharpness-aware minimization}'' (SAM) to simultaneously minimize loss value and loss sharpness, where the sharpness can be quantified as the maximized difference of loss when a perturbation is added to the current weights. In practice, SAM performs two forward-backward computations for each optimization step, where the first forward-backward is to obtain the perturbation for each model parameter and the second one is to update the parameters. Many prior works~\cite{wu2020adversarial,zheng2021regularizing} show the effectiveness of SAM in the vision domain, motivated by this, \citet{bahri-etal-2022-sharpness} first apply the SAM to the language domain, more recently.

Although \citet{bahri-etal-2022-sharpness} empirically show the remarkable performance of SAM on several language understanding tasks, SAM calculates perturbations indiscriminately for all parameters, which is time-consuming and hinders the application of SAM. Furthermore, inspired by the finding~\cite{keskar2016large} that only about 5\% of parameters are sharp and rise steeply during optimization, we notice that \textbf{\textit{not all parameters contribute equally to the optimization of training}}.
Hence, this raises a question that \textit{whether we can calculate perturbations for only some individual parameters, and thus make the optimizer focus on these important parameters}.

To this end, we propose a novel optimization approach, Fisher SAM (FSAM), which introduces a Fisher mask to improve the efficiency and effectiveness of SAM. In short, FSAM first uses the Fisher information~\cite{fisher1922mathematical} as the metric to identify the sharper parameters\footnote{~We refer to these parameters as the important ones, because they will rise steeply during optimization and affect the model generalization significantly.} and formulates a binary Fisher mask correspondingly. Then, the Fisher mask is multiplied with the perturbations to obtain the sparse perturbations, which are lastly used to perform regularization in the parameter update. In this way, only parts of sharper parameters will be added into the perturbations, and the optimizer can thus focus more on these important parameters.  
Also, the sparse perturbations could ensure the training acceleration via sparse back-propagation\footnote{~Since the fine-grained sparse training is limited to the hardware, we do not achieve actual sparse speedup in this work. Despite it, we still believe that FSAM has great potential to achieve true training acceleration in the future, with the development of hardware for fine-grained sparse operation.}.
Moreover, one may concern that the sparse Fisher mask would affect the convergence rate of FSAM~\cite{lin2019dynamic}. Hence, we theoretically provide the convergence analysis of FSAM, ensuring that the convergence of FSAM is irrelevant to the Fisher mask.

We conduct a large-scale and systematic study to evaluate the performance and effectiveness of FSAM. Firstly, we apply SAM and FSAM to fine-tune various PLMs on parts of GLUE and SuperGLUE benchmarks, where the results show that FSAM consistently outperforms the vanilla SAM by 0.67$\sim$1.98 average score among these PLMs, and surpasses the Adam~\cite{kingma2015adam} optimizer by 1.41$\sim$1.91 points. Secondly, we conduct experiments on two popular generation tasks (\textit{i.e.}, XSUM and CoNLL2014) and prove that FSAM can deliver promising results against SAM.
Lastly, quantitative analysis and in-depth discussion demonstrate the universality and effectiveness of FSAM in various complex scenarios, and prove that FSAM indeed brings better model generalization. Specifically, we show that our Fisher mask strategy not only works well in the SAM, but also can be applied to other SAM variants.

To summarize, our contributions are two-fold: (1) We propose a novel optimization approach (namely FSAM) with theoretical convergence guarantee for PLMs. Specifically, FSAM improves the performance and efficiency of recently-proposed SAM via a Fisher mask strategy, which can also be applied to more SAM variants.
(2) Extensive experiments show that FSAM consistently outperforms the SAM by a large margin on both language understanding and generation tasks. The systematic study demonstrates the effectiveness and universality of FSAM on improving model generalization.


%% file: section/related_work.tex
\section{Related Work}
\label{sec_related}

\paragraph{SAM and its variants.} \citet{hochreiter1994simplifying} first show the strong correlation between the flat minima and the generalization of a model, inspired by this, \citet{foret2020sharpness} propose the SAM to find a flat minimum and thus improve model generalization. 
While many existing works prove the effectiveness of SAM on various computer vision tasks~\cite{wu2020adversarial,chen2021vision,zheng2021regularizing}, the double forward-propagation process of SAM brings more computational cost. To this end, \citet{du2021efficient} propose an Efficient SAM (ESAM) for reducing the computational cost of SAM. Additionally, there are also some efforts that focus on more efficient and effective SAM optimization~\cite{zhuang2021surrogate,kwon2021asam,ssam}.

\paragraph{Improving Generalization.}
Recently, we have witnessed numerous PLMs that achieved tremendous success in the community of NLP~\cite{yang2019xlnet,devlin2019bert,brown2020language,lewis2020bart,raffel2020exploring, joshi2020spanbert,he2020deberta,qi2021prophetnet,zhong2022e2s2}. The current dominant fine-tuning approach needs to tune all pretrained parameters for each downstream task, which makes the PLM easily memorize the training data and thus leads to overfitting. To tackle this issue, some works attempt to provide implicit and explicit regularization into the training of models, such as dropout~\cite{srivastava2014dropout}, label smoothing~\cite{muller2019does}, mixup~\cite{zhang2018mixup} and other data-augmentation methods~\cite{sennrich2016improving,wang2018switchout,zhong2021joint,wang2022contrastive,ding2022redistributing}.
On the other hand, motivated by the successful applications of SAM in the vision domain, \citet{bahri-etal-2022-sharpness} involve applying SAM to optimize the T5~\cite{raffel2020exploring} model on multiple language tasks and show that SAM can improve the generalization of PLMs effectively. 

We depart from the prior work~\cite{bahri-etal-2022-sharpness} and ours as follows: 1) \textit{different motivations}: instead of verifying the effect of vanilla SAM on several language understanding tasks, we aim to improve the efficiency and effectiveness of SAM. 2) \textit{different contributions}: our main contribution is to propose a fisher mask strategy, which can be applied to both SAM and its variants. 3) \textit{more analysis}: we provide more experimental results and analysis towards the effectiveness of our method in more complex scenarios. 


%% file: section/method.tex
\section{Methodology}
In this section, we first review the Sharpness-Aware Minimization, and then propose our Sharpness-Aware Minimization with Fisher mask, coined as FSAM. Finally, we theoretically analyze the convergence of FSAM with adaptive learning rate.

\subsection{Sharpness-Aware Minimization}
\paragraph{Preliminary.}
In this paper, we denote the weight of a neural network as $\boldsymbol{w}\in \mathbb{R}^d$. Suppose the training dataset $\mathcal{S}=\{(x_i, y_i)\}^n_{i=1}$~\emph{i.i.d.} drawn from the distribution $\mathcal{D}$. The object function of the data $x_i$ from $\mathcal{S}$ is denote as $f_{\mathcal{S}}(x_i)$. Since the Adam~\cite{kingma2015adam} and its variants are widely used in NLP tasks, the learning rate is estimated via RMSProp/Adam style.

\paragraph{Sharpness-Aware Minimization.} \citet{foret2020sharpness} propose the Sharpness-Aware Minimization~(SAM) to improve the generalization, which is achieved by the following min-max problem:
\begin{equation}
    \min_{\boldsymbol{w}} \max_{||\boldsymbol{\epsilon}||_2\leq \rho} f(\boldsymbol{w}+\boldsymbol{\epsilon}),
\end{equation}
where $\rho$ is a predefined value to control the neighborhood size, and the $\boldsymbol{\epsilon}$ is the perturbation vector on model weight. The optimization is expected that the model loss will not significantly rise with a certain amount of weight change controlled by $\rho$, which is intuitively consistent with the generalization capacity of model.

With the Taylor expansion, the perturbation vector $\boldsymbol{\epsilon}$ could be achieved approximately:
\begin{align}
\label{equ:epsilon}
    \boldsymbol{\epsilon^*} =& \mathop{\arg\max}_{||\boldsymbol{\epsilon}||_2\leq\rho} f_{\mathcal{S}}(\boldsymbol{w} + \boldsymbol{\epsilon}) 
    \\
    \approx& \mathop{\arg\max}_{||\boldsymbol{\epsilon}||_2\leq\rho} f_{\mathcal{S}}(\boldsymbol{w}) + \boldsymbol{\epsilon} \cdot \nabla_{\boldsymbol{w}}f(\boldsymbol{w}) 
    \\
    =&
    \rho \cdot {\nabla_{\boldsymbol{w}}f(\boldsymbol{w})}{\big /}{||\nabla_{\boldsymbol{w}}f(\boldsymbol{w})||_2},
\end{align}
and the object function could be simplified as \begin{equation}
    \min_{\boldsymbol{w}}f(\boldsymbol{w} + \rho \frac{\nabla_{\boldsymbol{w}}f(\boldsymbol{w})}{||\nabla_{\boldsymbol{w}}f(\boldsymbol{w})||_2}),
\end{equation}
The solution of the above function could be obtained by a two-step gradient descent. In the first gradient descent step, the perturbation vector $\boldsymbol{\epsilon}$ is calculated by Equation~\ref{equ:epsilon}. The second gradient descent step is the actual weight update. 

However, despite the improvement of SAM on many tasks, SAM requires a two-step gradient calculation which leads to the double overhead compared to the conventional optimizer,~\emph{e.g.}, Stochastic Gradient Descent~(SGD) and Adam.

\subsection{Sharpness-Aware Minimization with Fisher Mask}
In this subsection, we propose the Sharpness-Aware Minimization with Fisher Mask~(FSAM) in detail, which reduces the computation of SAM by sparse calculation.

To be specific, we compute only a fraction of the elements in the perturbation vector $\boldsymbol{\epsilon}$, which would be multiplied by a sparse binary mask $\boldsymbol{m}\in \{0, 1\}^d$. To control the amount of perturbation, the sparse mask $\boldsymbol{m}$ satisfies $\boldsymbol{1}^T\boldsymbol{m}=(1-s)\cdot d$, where the $s$ is the predefined sparse ratio and empirically set to 0.9. The objective function of FSAM is denoted as
\begin{equation}
    \min_{\boldsymbol{w}} f_{\mathcal{S}}({\boldsymbol{w}+\rho \frac{\nabla_{\boldsymbol{w}}f(\boldsymbol{w})\odot \boldsymbol{m}}{||\nabla_{\boldsymbol{w}}f(\boldsymbol{w})||_2})},
\end{equation}
where $\odot$ is the Hadamard product,~\emph{i.e.}, the element-wise multiplication. For the stability of optimization, we update the mask $\boldsymbol{m}$ with a fixed interval (denoted as \gap) during training. The algorithm of FSAM is shown in Algorithm~\ref{alg:fsam}.

\begin{algorithm}[ht]
\small
    \caption{Fisher SAM (FSAM)}
    \label{alg:fsam}
\begin{algorithmic}[1]
\REQUIRE sparse ratio $s$, dense model $\boldsymbol{w}$, binary mask $\boldsymbol{m}$, update interval $T_m$, base learning rate $\gamma$, $\hat{v}_{-1}=\delta^2$, training set $\mathcal{S}$.
\STATE Initialize $\boldsymbol{w}$ and $\boldsymbol{m}$ randomly.
\FOR{epoch $t=1,2 \ldots T$}
    \FOR{each training iteration}
    \STATE{Sample a batch from $\mathcal{S}$: $\mathcal{B}$}
    \STATE{Compute perturbation $\boldsymbol{\epsilon}$ by Eq.~\ref{equ:epsilon}}
    \IF{$t \mod T_m = 0$}
        \STATE{Sample $N_{Fisher}$ data from distribution $\mathcal{S}$.}
        \STATE{Compute Empirical Fisher by Equation~\ref{equ:empirical fisher}.}
        \STATE{$\boldsymbol{m_1}\!\gets {\rm ArgTopK}(\hat{F}, (1-s)\cdot|\boldsymbol{w}|)$}
        \STATE{$\boldsymbol{m_0}\!\gets {\rm ArgTopK}(-\hat{F}, s\cdot|\boldsymbol{w}|)$}
        \STATE{Update mask $\boldsymbol{m}$ by merging: $\boldsymbol{m}=\boldsymbol{m_0}\cup\boldsymbol{m_1}$.}
    \ENDIF
    \STATE{$\boldsymbol{\epsilon}\gets\boldsymbol{\epsilon} \odot \boldsymbol{m}$}
    \ENDFOR
\STATE{Compute SAM gradient $g_t=\nabla f_{\mathcal{B}}(\boldsymbol{w}+\boldsymbol{\epsilon})$}
\STATE{$v_t = \beta_2 v_{t-1} + (1-\beta_2) [g_t]^2$}
\STATE{$\hat{v}_t=\max(\hat{v}_{t-1},v_t)$}
\STATE{$\boldsymbol{w} \gets \boldsymbol{w} - \gamma \nabla g_t\odot \frac{1}{\sqrt{\hat{v}_t}}$}
\ENDFOR
\RETURN{Final weight of model $\boldsymbol{w}$}
\end{algorithmic}
\end{algorithm}

To find the optimal mask during training, we apply the Fisher information to achieve sparse perturbation. The Fisher information is proposed by~\cite{fisher1922mathematical} to measures the information carried by an observable random variable about the unknown parameters of the distribution. The Fisher information is defined by

\begin{equation}
\label{equ:fisher}
F = \mathbb{E}_{x}\left[ \mathbb{E}_{y}\nabla \log p(y|x) \nabla \log p(y|x)^T\right],
\end{equation}
where the $p(y|x)$ is the output of model in machine learning. However, due to the over-parameterized model in deep learning, the computation of Fisher information is unacceptable,~\emph{i.e.}, $F\in\mathbb{R}^{|\boldsymbol{w}|\times|\boldsymbol{w}|}$. To save the computational effort, we approximate Fisher information as the diagonal matrix,~\emph{i.e.}, $F\in \mathbb{R}^{|\boldsymbol{w}|}$. Consider the expectation in Equation~\ref{equ:fisher}, the first one is the data distribution $x\sim p(x)$, which is not available in most tasks. We approximate it by sampling $N_{Fisher}$ data from $p(x)$:
\begin{equation}
    F=\frac{1}{N_{Fisher}} \mathbb{E}_{y}\nabla \log p(y|x_i)^2.
\end{equation}
The second expectation is over $p(y|x)$, which can be achieved by the label $y_i$ for data $x_i$ in supervised learning. Finally, we calculate the Fisher information as "Empirical Fisher":
\begin{equation}
\label{equ:empirical fisher}
    \hat{F} = \frac{1}{N_{Fisher}} \nabla \log p(y_i|x_i)^2.
\end{equation}
Since the empirical Fisher is the same size as the weight,~\emph{i.e.}, $\hat{F}\in\mathbb{R}^{|\boldsymbol{w}|}$, the value of the element in Fisher $\hat{F}$ represents the importance of the corresponding element in weight $\boldsymbol{w}$. Thus, we sort the elements of $\hat{F}$ in descending, and the weights with top $k$ Fisher values will be perturbed,~\emph{i.e.}, the corresponding element in mask will be set to 1:
\begin{equation}
    \boldsymbol{m_1}\!\gets {\rm ArgTopK}(\hat{F}, (1-s)\cdot|\boldsymbol{w}|),
\end{equation}
where $\boldsymbol{m_1}$ is the set whose elements in the mask $\boldsymbol{m}$ are 1,~\emph{i.e.}, $\boldsymbol{m}\!=\!\{m_i\!=\!1|m_i\!\in\! \boldsymbol{m}\}$, and $\rm ArgTopK$ $(x,k)$ returns the top $k$ largest values among $x$. On the other hand, the other weights with small Fisher values will not be perturbed,~\emph{i.e.}, the corresponding element in mask will be set to 0:
\begin{equation}
    \boldsymbol{m_0}\!\gets {\rm ArgTopK}(\hat{F}, s\cdot|\boldsymbol{w}|).
\end{equation}

\subsection{Theoretical Analysis}
In this subsection, we theoretically analyze the convergence and generalization of FSAM. Due to the space limitation, we only show the convergence analysis here, and the generalization analysis and whole proof are presented in Appendix~\ref{appendix_proof}.
\begin{assumption}
\label{assume:l-smooth}
($L$-smooth.) Consider $f$ is differentiable with gradient Lipschitz property: It exists $L > 0$ s.t. 
\begin{equation*}
    ||\nabla f(w) - \nabla f(v)|| \leq L||w - v|| , \forall  w, v \in \mathbb{R}^d.
\end{equation*}

\end{assumption}

\begin{assumption}
\label{assume:bounded-variance}
(Bounded stochastic gradients.) The variance of stochastic gradient is bounded:
\begin{equation*}
    \mathbb{E}[||\nabla f_i(x)-\nabla f(x)||^2]\leq \sigma^2
\end{equation*}
\end{assumption}

\begin{assumption}
\label{assume:bounded-gradient}
(Bounded gradient.) The stochastic gradient is bounded: It exists $G \geq 0$ s.t. 
\begin{equation*}
    ||\nabla f_i(w) ||_{\infty}\leq G
\end{equation*}

\begin{theorem}
\label{con}
Consider the function $f$ under the assumption~\ref{assume:l-smooth},\ref{assume:bounded-variance},\ref{assume:bounded-gradient}, and a fixed base learning rate $\gamma_t$ satisfies that $\gamma_t \leq \frac{\delta}{8L}$, we have
\end{theorem}
\begin{align*}
    &\frac{1}{T}\sum_{t=0}^{T-1}\mathbb{E}||\nabla f(x_t)||^2 \leq \frac{2G  f(x_{0})-f^*}{\gamma_tT} \\& + \frac{20GL^2\rho^2}{\delta} + \frac{2G^3}{T}d(\frac{1}{\delta}-\frac{1}{G}) \\& + \frac{4G\gamma_tL}{\delta}\frac{L\rho^2}{\delta}+ \frac{4G\gamma_tL}{\delta} \frac{\sigma^2}{b\delta} \\&  + \frac{4\gamma_tLG^3}{T}d(G^2-\delta^2)
\end{align*}
The Theorem~\ref{con} shows that when $T$ is large, FSAM could achieve the linear speedup convergence rate with respect to mini-batch size $b$ under the setting of $\gamma_t = O(\sqrt{\frac{b}{T}})$ and $\rho=O(\sqrt{\frac{1}{bT}})$, \emph{i.e.},
\begin{align*}
    \frac{1}{T}\sum_{t=0}^{T-1}\mathbb{E}||\nabla f(x_t)||^2 = O(\sqrt{\frac{1}{bT}})
\end{align*}
\end{assumption}

%% file: tables/main_results.tex
\begin{table*}[ht]
\caption{Experimental results (dev scores) on various language understanding benchmarks. Comparison between vanilla SAM and our proposed FSAM applied to four widely used large-scale PLMs. The best results for each setting are in \textbf{bold}. ``AVG.'' denotes the average scores on all tasks, which are \underline{underlined}. Results show that our FSAM brings consistent improvements across all understanding tasks among different PLMs.}
\label{tab_main}
\centering
\scalebox{0.9}{
\begin{tabular}{l|cccccccccc|l}
\toprule
\multicolumn{1}{c|}{\multirow{2}{*}{\textbf{Method}}} &
 CoLA                 &
 \multicolumn{2}{c}{MRPC}                    &
 \multicolumn{2}{c}{STS-B}                   &
 RTE                  &
 \multicolumn{1}{c}{CB}     &
 \multicolumn{1}{c}{BoolQ}    &
 \multicolumn{1}{c}{WSC}           &
 \multicolumn{1}{c}{WiC}          &
 \multicolumn{1}{|c}{\multirow{2}{*}{AVG.}} \\ \cmidrule(lr){2-7} \cmidrule(lr){8-11}
\multicolumn{1}{c|}{}                       &
 \textit{Mcc.}                 &
 \textit{Acc.}                 &
 \textit{F1.}                  &
 \textit{Pear.}              &
 \textit{Spea.}              &
 \textit{Acc.}                 &
 \textit{Acc.}          &
 \textit{Acc.}          &
 \textit{Acc.}          &
 \textit{Acc.}          &
 \multicolumn{1}{c}{}                      \\
 \hline \hline
\multicolumn{12}{c}{\cellcolor{lightgray}{ BERT-large}} \\
Adam                                       &
62.8 &
87.3                  &
91.1 &
89.5 &
89.3 &
70.7 &
87.5&
74.3&
68.3&
72.7&
\underline{79.35}                                          \\
Adam+SAM                                   &
 62.1 &
87.9&
91.4 &
89.8 &
89.4 &
71.5 &
91.1&
72.9&
68.3&
74.1&
\underline{79.85}                                         \\
Adam+FSAM                       &
\textbf{63.4} &
\textbf{89.0} &
\textbf{92.0} &
\textbf{90.4} &
\textbf{89.9} &
\textbf{74.4} &
\textbf{94.6}&
\textbf{75.3}&
\textbf{68.5}                                   &
\textbf{74.4}&
\underline{\textbf{81.19}}            \\ \midrule
\multicolumn{12}{c}{\cellcolor{lightgray}{ELECTRA-large}}  \\
Adam                                       &
69.0 &
89.2 &
92.4 &
92.1 &
92.1 &
87.3 &
91.1&
85.6&
83.6&
\textbf{74.4}&
\underline{85.68}                                           \\
Adam+SAM                                   &
63.9 &
91.9 &
94.2 &
\textbf{92.4} &
92.4 &
\textbf{89.2} &
92.9&
82.2&
84.6&
72.4&
\underline{85.61}                                           \\
Adam+FSAM                       &
\textbf{69.6} &
\textbf{92.4} &
\textbf{94.5} &
92.3 &
\textbf{92.5} &
88.8 &
\textbf{96.4}&
\textbf{85.9}&
\textbf{89.4}&
74.1&
\underline{\textbf{87.59}}            \\ \midrule
\multicolumn{12}{c}{\cellcolor{lightgray}{ALBERT-xxlarge}}  \\
Adam                                       &
 71.1                     &
 90.7                     &
  93.3                    &
  92.9                    &
  92.7                    &
  87.0 &
89.3&
86.8&
85.6&
75.5&
\underline{86.49}                                          \\
Adam+SAM                                   &
 69.9                     &
 90.7                     &
  93.2                    &
 92.6                     &
 92.4                     &
 88.1                     &
 \textbf{91.1} &
 87.7&
  82.7&
 \textbf{76.6}&
 \underline{86.50}                                          \\
Adam+FSAM                       &
  \textbf{72.3}                    &
 \textbf{91.9}                     &
  \textbf{94.2}                    &
  \textbf{93.0}                    &
 \textbf{92.8}                     &
  \textbf{88.8}                    &
  \textbf{91.1} &
  \textbf{87.9}&
  \textbf{86.5}&
  \textbf{76.6}&
  \underline{\textbf{87.51}}          \\ \midrule
\multicolumn{12}{c}{\cellcolor{lightgray}{RoBERTa-large}}     \\ 
Adam                                       &
 66.7                 &
 90.4                 &
 93.1                 &
 \textbf{92.1}                 &
 \textbf{92.0}                   &
 87.0                   &
  92.8&
86.0 &
 78.1&
 73.3&
 \underline{85.15}             \\
Adam+SAM                                   &
 68.5                 &
 \textbf{90.7}        &
 \textbf{93.3}        &
 91.5                 &
 91.3                 &
 \textbf{87.7}        &
 96.4&
 84.2&
 81.3&
 74.0&
 \underline{85.89}                       \\
Adam+FSAM                       &
 \textbf{69.5}        &
 \textbf{90.7}        &
 93.2                 &
 91.9        &
 91.6        &
 \textbf{87.7}        &
 \textbf{98.2}&
 \textbf{86.8}&
 \textbf{81.5}&
 \textbf{74.5} &
  \underline{\textbf{86.56}}      
 \\ \bottomrule 
\end{tabular}
}
\end{table*}

%% file: tables/main_results2.tex
\begin{table*}[ht]
\caption{Experimental results on two popular generation tasks. We use the representative sequence-to-sequence PLM (BART) in this study. It shows that our FSAM works well on the language generation tasks as well. ``\sig'' indicates that FSAM is significantly better than baselines at significance level p<0.05.}
\label{tab_main2}
\centering
\scalebox{0.9}{
\begin{tabular}{l|ccc|ccc|c}
\toprule
\multicolumn{1}{c|}{\multirow{2}{*}{\textbf{Method}}} & 
\multicolumn{3}{c}{XSUM}                         & 
\multicolumn{3}{c}{CoNLL2014}                          & 
\multicolumn{1}{|c}{\multirow{2}{*}{AVG.}}
\\
\cmidrule(lr){2-4} \cmidrule(lr){5-7}
\multicolumn{1}{c|}{}                       &
Rouge\_1       &
Rouge\_2       &
Rouge\_L       &
Precision              &
Recall              &
F\_0.5         &
\\

\hline \hline
\multicolumn{8}{c}{\cellcolor{lightgray}{BART-large}} \\

Adam                                       &
44.35          &
21.66          &
36.62          &
52.94          &
41.18          &
50.08          &
\underline{41.14}                 \\

Adam+SAM                                   &
44.81          &
21.95          &
36.97          &
53.52          &
41.76          &
50.70          &
\underline{41.62}                 
\\

Adam+FSAM                      &
\textbf{45.03\sig} &
\textbf{22.15\sig} &
\textbf{37.14\sig} &
\textbf{54.33\sig} &
\textbf{42.15\sig} &
\textbf{51.36\sig} &
\underline{\textbf{42.03}}       
\\
\midrule
\multicolumn{8}{c}{\cellcolor{lightgray}{BART-base}} 
\\

Adam                                       &
39.38          &
17.21          &
31.93          &
43.27          &
34.11          &
41.06          &
\underline{34.49}                 \\

Adam+SAM                                   &
40.38          &
18.00          &
33.00          &
50.39          &
33.51          &
45.78          &
\underline{36.84}                 \\

Adam+FSAM                                  &
\textbf{40.60\sig} &
\textbf{18.31\sig} &
\textbf{33.28\sig} &
\textbf{51.77\sig} &
\textbf{34.04\sig} &
\textbf{46.89\sig} &
\underline{\textbf{37.48}}         
\\ \bottomrule

\end{tabular}
}
\end{table*}

%% file: tables/smaller_model.tex
\begin{table}[t]
\caption{Results of smaller PLMs with different optimizers on parts of understanding tasks. BERT-base and RoBERTa-base are used in this experiment.}
\label{tab_main3}
\centering
\scalebox{0.85}{
\begin{tabular}{l|cccccc}
\toprule
\multicolumn{1}{c|}{\multirow{2}{*}{\textbf{Method}}} &
 CoLA                 &
 \multicolumn{2}{c}{MRPC}                    &
 \multicolumn{2}{c}{STS-B}                   &
 RTE                 
 \\ \cmidrule(r){2-7}
\multicolumn{1}{c|}{}                        &
 \textit{Mcc.}                 &
 \textit{Acc.}                 &
 \textit{F1.}                   &
 \textit{Pear.}                &
 \textit{Spea.}                &
 \textit{Acc.}                 \\
 \hline \hline
\multicolumn{7}{c}{\cellcolor{lightgray}{ BERT-base}} \\
Adam                          &
 \textbf{54.3}&
 85.8&
 90.0&
 89.2&
 88.9&
 68.4
\\
-w SAM                        &
 53.0&
 87.3&
 90.9&
 89.3&
 89.1&
 66.4
\\
-w FSAM                       &
 53.8&
 \textbf{87.7}&
 \textbf{91.3}&
 \textbf{89.5}&
 \textbf{89.2}&
 \textbf{70.0}
\\ \midrule
\multicolumn{7}{c}{\cellcolor{lightgray}{RoBERTa-base}}  \\

Adam                          &
61.3          & 
87.5          & 
90.6          & 
90.6          & 
90.4          & 
78.3          
 \\
-w SAM                        &
60.6          & 
89.2          & 
92.1          & 
90.5          & 
\textbf{90.4} & 
78.7
 \\
-w FSAM                       &
\textbf{61.4} & 
\textbf{89.5} & 
\textbf{92.5} & 
\textbf{90.7} & 
\textbf{90.4} & 
\textbf{80.1}
 \\ \bottomrule 
\end{tabular}
}
\end{table}

%% file: section/experimental_setup.tex
\section{Experimental Setup}
\label{sec_setup}
\subsection{Tasks and Datasets}
To investigate the effectiveness and universality of our FSAM method, we conduct extensive experiments on various NLP tasks. Specifically, different from \citet{bahri-etal-2022-sharpness} that only verify the method on several language understanding tasks, we evaluate our method on both language understanding and generation tasks.

\paragraph{Language Understanding Tasks.}
Following many previous works~\cite{vu2021spot,bahri-etal-2022-sharpness,zhong2022e2s2}, we conduct experiments on a combination of tasks from GLUE~\cite{wang2018glue} and SuperGLUE~\cite{wang2019superglue} benchmarks, including linguistic acceptability (CoLA), natural language inference (RTE, CB), paraphrase and similarity (MRPC and STS-B), question answering (BoolQ), word sense disambiguation (WiC) and coreference resolution (WSC).
In practice, we evaluate the performance with Accuracy (``\textit{Acc.}'') metric for most tasks, except the additional F1 score for MRPC, the Pearson-Spearman correlations (``\textit{Pear./Spea.}'') for STS-B and the Matthew correlation (``\textit{Mcc.}'') for CoLA.

\paragraph{Language Generation Tasks.}
We also use two popular generation tasks following~\citet{liu2021understanding,zhang2022bliss} as the benchmarks, \textit{i.e.}, abstractive summarization (XSUM) and grammatical error correction (CoNLL2014). For the XSUM, we report results in terms of standard ROUGE metrics~\cite{Lin:04}, \textit{i.e.}, Rouge-1, Rouge-2 and Rouge-L, respectively. For the CoNLL2014, MaxMatch scores~\cite{Dahlmeier:12} are used for evaluation with Precision, Recall, and $F_{0.5}$ values~\footnote{~Due to the space limitation, we present the details of all used tasks and datasets in Appendix~\ref{appendix_data}}.

\subsection{Implementations}
In practice, we use the pretrained models and code in HuggingFace\footnote{~\url{https://github.com/huggingface/transformers}}~\cite{wolf2019huggingface}. Specifically, for the understanding tasks, we employ 4 widely used PLMs in our study, \textit{i.e.}, BERT~\cite{devlin2019bert}, ELECTRA~\cite{clark2019electra}, ALBERT~\cite{lan2019albert} and RoBERTa~\cite{liu2019roberta}. Furthermore, an representative sequence-to-sequence model, BART~\cite{lewis2020bart}, is used for the generation tasks.

We compare our proposed FSAM method with the base optimizer (without using any SAM approach) and vanilla SAM method.
Specifically, the Adam~\cite{kingma2015adam} is used as the base optimizer to tune our models. 
The $\beta_2$ and weight decay of Adam are set as 0.999 and 0.01. SAM and FSAM use the same settings as above. More specially, we grid search for the neighborhood size of SAM and FSAM on \{1e-2, 5e-3, 1e-3\}. Additionally, for each downstream task, we follow the same hyper-parameter settings from the prior works~\cite{lewis2020bart,xu2021raise}. The detailed hyper-parameters of fine-tuning on these downstream tasks can be seen in Appendix~\ref{appendix_parameters}. We report the averaged results over 5 random seeds for NLU tasks, while for NLG tasks, we follow existing works~\cite{collins2005clause,ding2021progressive} and use the Bootstrap test~\cite{berg2012empirical} to calculate the statistical significance.


%% file: section/main_results.tex
\section{Main Results}
\paragraph{FSAM outperforms vanilla SAM by a large margin across different PLMs.}
Table~\ref{tab_main} shows the results of all understanding tasks. We can observe that SAM achieves better average scores than the base Adam in most scenarios, confirming the effectiveness of SAM in improving generalization~\cite{bahri-etal-2022-sharpness}. Moreover, with the help of our Fisher mask strategy, FSAM consistently improves the vanilla SAM by a large margin across all PLMs. Specifically, FSAM yields an improvement of up to 1.98 average score on ELECTRA, 1.01 average score on ALBERT and 1.34 average score on BERT. The average improvement on RoBERTa is slight but also higher than 0.67.

\paragraph{FSAM also works well on the generation tasks.}
Prior works~\cite{kwon2021asam, bahri-etal-2022-sharpness}, which involve the study of SAM or its variants, usually conduct experiments on the image or text classification tasks, \textit{e.g.}, CIFAR-10~\cite{krizhevsky2009learning} and ImageNet~\cite{krizhevsky2012imagenet}. The effectiveness of optimizer on other types of tasks, \textit{e.g}, generation tasks in NLP, has not been explored well. Thus far, we evaluate our FSAM on the generation tasks and present the results in Table~\ref{tab_main2}. 
It can be seen that FSAM can deliver promising results against the vanilla SAM as well. Note that both FSAM and SAM outperform the base Adam optimizer, indicating the applicability of SAM and its variants on generation tasks. 

\paragraph{FSAM improves performance on various model sizes.}
To investigate whether our FSAM is helpful for various scales of PLMs, we evaluate the performance on smaller PLMs, \textit{i.e.}, BERT-base, RoBERTa-base and BART-base. The results are showed in Table~\ref{tab_main3} and Table~\ref{tab_main2}, respectively. We can see that FSAM consistently outperforms the vanilla SAM on multiple smaller PLMs, to be specific, the relative improvements of BERT-base and BART-base are up to 0.92 and 0.64 average scores. These results prove that FSAM works well on various model sizes.  

%% file: section/analysis.tex
\section{Analysis and Discussion}
\label{sec_analysis}
In this section, we examine whether our approach works in more complicated scenarios, and provide a more intuitive comparison between different optimizers towards the generalization. More analysis and results can be found in Appendix.

\input{figures/sparse_ratio}

\subsection{Parameter Analysis}
There are two important hyper-parameters (\textit{i.e.}, $s$ and \gap) in our FSAM, where the $s$ refers to the sparse ratio and \gap is used to control the update frequency of Fisher mask. 
Here, we evaluate the performance of FSAM with different $s$ and \gap on several downstream tasks to analyze their effects.

Firstly, Figure~\ref{fig:sparse} shows the results based on different $s$. We can observe FSAM outperforms the vanilla SAM and base Adam in most settings, indicating the robustness of FSAM. Specifically, when the sparse ratio is 0.9, FSAM consistently achieves the best performance on both tasks. Secondly, for \gap, we show the performance of FSAM on different \gap in Table~\ref{tab_para2}. Too small \gap (\textit{e.g.}, 10) may lead to the Fisher mask updating too fast, thus affecting the stability of model optimization. Recall that we set $s=0.9$ and $\gap=100$ as the default setting.

\input{tables/other_base}

\input{tables/other_sam}

\subsection{Complementarity with Other Optimizers}
As aforementioned, we show the effectiveness of our Fisher mask strategy on SAM optimization. To further prove the universality of our proposed strategy, we examine whether the strategy is complementary with i) \textbf{\textit{more base optimizers}} and ii) \textbf{\textit{other efficient SAM variants.}}

To verify i),  we use the additional AMSGrad and Adagrad as the base optimizers and evaluate the performance with different strategies, respectively. Table~\ref{tab_analysis2} lists the results of RoBERTa-large. It can be seen that FSAM consistently achieves the best performance upon these base optimizers, showing \textit{our strategy is not sensitive to the base optimizers.} 

For ii), we apply our strategy to another two cutting-edge SAM-variant optimizers, \textit{i.e.}, ESAM~\cite{du2021efficient} and GSAM~\cite{zhuang2021surrogate}. Table~\ref{tab_analysis1} shows the results, where F\_ESAM and F\_GSAM refer to the optimizations using our strategy. When evaluating RoBERTa-large on these tasks, compared to the vanilla ESAM and GSAM, our method can bring a 0.70 average score improvement. \textit{This indicates that our Fisher mask strategy is not only beneficial to the vanilla SAM, but also can be applied to other efficient SAM variants.}

\input{figures/training_sizes}

\subsection{Results in Low-resource Scenarios}
Prior works~\cite{chen2021vision,bahri-etal-2022-sharpness} show that SAM helps more when there is less training data. Here, we verify how our Fisher mask strategy affects the effectiveness of SAM in low-resource scenarios. In practice, we follow~\citet{bahri-etal-2022-sharpness} and sub-sample the training splits for several GLUE datasets at rates ranging from 10\% to 90\%. Notably, due to the space limitation, we only report parts of results on BERT-large and RoBERTa-large in Figure~\ref{fig:training_size}.

We can observe consistent gains from both SAM and our FSAM across all sizes of sub-sampled training sets, which confirms the statement in prior work~\cite{bahri-etal-2022-sharpness}. Moreover, it can also be seen that our FSAM improves the vanilla SAM by a large margin in low-resource scenarios, especially when there is only 20\% training data. More specifically, when fine-tuning the RoBERTa-large on the STS-B dataset, the relative improvements of FSAM are up to 15.0 and 15.1 in terms of accuracy and F1 score, respectively. \textit{These results show that our method is more helpful in low-resource scenarios.}

\input{figures/task_generalization}
\input{figures/3d-loss}
\input{figures/2d-loss}
\subsection{Does FSAM Bring Better Generalization?}
We prove the effectiveness of our FSAM by large-scale experiments as above. Here, to examine whether FSAM indeed brings better generalization, we i) \textbf{\textit{measure the generalization properties}} (\textit{i.e.}, task generalization) of different optimizations, and ii) \textbf{\textit{visualize the generalization of models}} via the training loss landscapes.

\paragraph{Task Generalization.} 
The common wisdom is that models with better generalization would perform better on out-of-domain data~\cite{xu2021raise}. Thus, to measure the generalization ability of the model quantitatively, we follow the experiments from~\citet{xu2021raise} and evaluate the performance of various fine-tuned models on out-of-domain data. In practice, we first fine-tune RoBERTa-large on the QNLI task (one of GLUE tasks) and then transfer it to other tasks, \textit{i.e.}, CoLA, MRPC, STS-B and RTE. The results of different optimization strategies are illustrated in Figure~\ref{fig:task_transfer}. 

We can observe that FSAM consistently outperforms the base Adam and vanilla SAM on different transferred tasks. To be more specific, compared with vanilla SAM, our FSAM brings a 2.37 relative average improvement score on these tasks, indicating that \textit{our method helps more in improving the generalization of model.}

\paragraph{Visualization of Landscape.}
Here, we visualize the loss landscapes of RoBERTa-base model fine-tuned on CoLA with different optimizers. In practice, we follow~\citet{li2018visualizing,zan2022complementarity} and show the 3D loss surface results in Figure~\ref{fig:3d_loss} by sampling 25$\times$25 points in the range of [-1, 1] from random ``filter normalized'' directions~\cite{li2018visualizing}. Additionally, following~\citet{hao2019visualizing,he2021effectiveness}, we also plot the 1D loss curve in Figure~\ref{fig:1d_loss} by linear interpolation between the pretrained model weights before (denoted as $\theta_0$) and after (denoted as $\theta_1$) fine-tuning, \textit{i.e.}, ``$\theta_1+\alpha \cdot (\theta_1-\theta_0)$'', where $\alpha$ is a scalar parameter that is ranged from -1 to 1.
We can find that the landscape of FSAM is much flatter than both base Adam and SAM, especially in the area of low loss. \textit{These results prove that FSAM can smooth the loss landscape and improve the generalization of PLMs effectively.}

%% file: figures/sparse_ratio.tex
\begin{figure}[t]
	\centering
	\begin{minipage}[t]{0.22\textwidth}
		\centering
		\includegraphics[width=\textwidth]{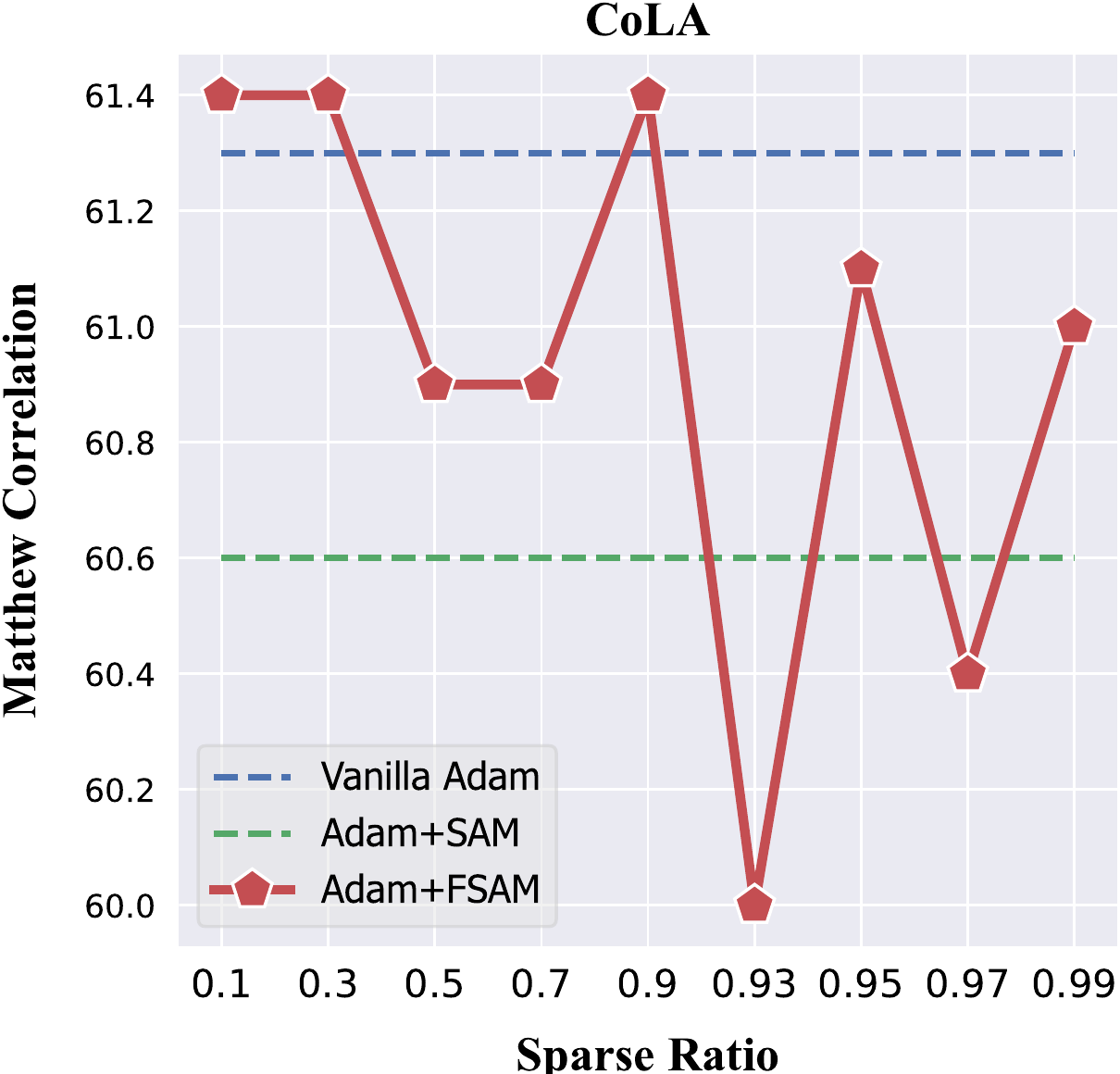}
		\end{minipage}%
		\hfill
		\hfill
		\begin{minipage}[t]{0.22\textwidth}
		\centering
		\includegraphics[width=\textwidth]{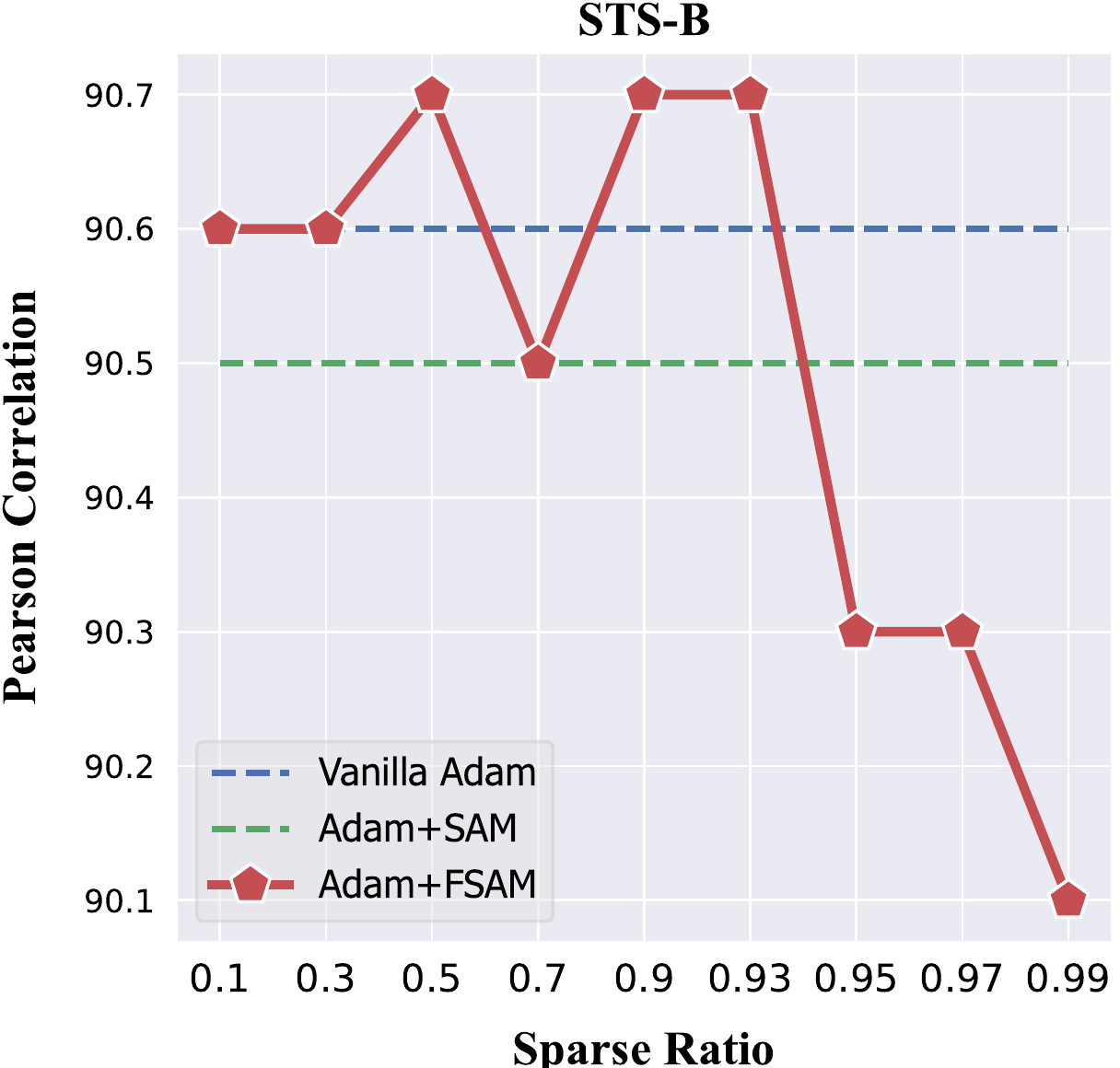} 
		\end{minipage}

	\caption{Results of FSAM at various sparse rates. RoBERTa-base models is used.
	}
	\label{fig:sparse}
\end{figure}

\begin{table}[h]
    \caption{Average performance (CoLA, MRPC, STS-B and RTE) of FSAM with different \gap, which denotes the fixed interval for updating Fisher mask.}
    \label{tab_para2}
    \centering
    \scalebox{0.8}{
    \begin{tabular}{lcccccc}
    \toprule
    \textbf{Method} & 10   & 50   & 100  & 200  & 500  \\ \midrule
    \textbf{FSAM-BERT$_{base}$}                        & 83.9 & 84.0 & \textbf{84.3} & 84.2 & 84.1\\
    \textbf{FSAM-RoBERTa$_{base}$} & 80.0 & 80.1 & \textbf{80.3} & 80.2 & 80.2  \\
    \bottomrule
    \end{tabular}
    }
\end{table}

%% file: tables/other_base.tex
\begin{table}[t]
\caption{Results of other base optimizers, \textit{i.e.}, AMSGrad and Adagrad. RoBERTa-large is used.}
\label{tab_analysis2}
\centering
\scalebox{0.8}{
\begin{tabular}{l|cccccc}
\toprule
\multicolumn{1}{c|}{\multirow{2}{*}{\textbf{Method}}} &
 CoLA                 &
 \multicolumn{2}{c}{MRPC}                    &
 \multicolumn{2}{c}{STS-B}                   &
 RTE                 
 \\ \cmidrule(r){2-7}
\multicolumn{1}{c|}{}                       &
 \textit{Mcc.}                 &
 \textit{Acc.}                 &
 \textit{F1.}                   &
 \textit{Pear.}                &
 \textit{Spea.}                &
 \textit{Acc.}                 \\
 \hline \hline
 \multicolumn{7}{c}{\cellcolor{lightgray}{ RoBERTa-large}} \\
AMSGrad                                       &
 66.8 &
 90.2 &
 92.7 &
 \textbf{91.7} &
 \textbf{91.5} &
 86.7
\\
-w SAM                                   &
 67.3 &
 90.2 &
 93.0 &
 91.6 &
 91.4 &
 87.0
\\
-w FSAM                       &
 \textbf{68.7} &
 \textbf{90.4} &
 \textbf{93.2}&
 91.2&
91.1&
 \textbf{87.7}
\\ \midrule
Adagrad                                       &
 \textbf{59.3} &
 81.1 &
 87.4 &
 \textbf{88.2} &
 \textbf{88.5} &
 80.9
\\
-w SAM                                   &
 57.5 &
 88.0 &
 91.4 &
 84.6 &
 85.7 &
 83.4
 \\
-w FSAM                       &
 57.5&
 \textbf{90.0}&
 \textbf{92.8}&
 86.8&
 87.1&
 \textbf{86.3}
 \\ \bottomrule 
\end{tabular}
}
\end{table}

%% file: tables/other_sam.tex
\begin{table}[t]
\caption{Results of some SAM variants (\textit{i.e.}, ESAM~\cite{du2021efficient} and GSAM~\cite{zhuang2021surrogate}) with our Fisher-masked strategy, denoted as ``F\_*''.}
\label{tab_analysis1}
\centering
\scalebox{0.8}{
\begin{tabular}{l|cccccc}
\toprule
\multicolumn{1}{c|}{\multirow{2}{*}{\textbf{Method}}} &
 CoLA                 &
 \multicolumn{2}{c}{MRPC}                    &
 \multicolumn{2}{c}{STS-B}                   &
 RTE                 
 \\ \cmidrule(r){2-7}
\multicolumn{1}{c|}{}                       &
 \textit{Mcc.}                 &
 \textit{Acc.}                 &
 \textit{F1.}                   &
 \textit{Pear.}                &
 \textit{Spea.}                &
 \textit{Acc.}                 \\
 \hline \hline
 \multicolumn{7}{c}{\cellcolor{lightgray}{ RoBERTa-large}} \\
Adam                                       &
 66.7 &
 90.4 &
 93.1 &
 92.1 &
 92.0 &
 87.0
\\
-w ESAM                                   &
 \textbf{68.5} &
 90.7 &
 93.3 &
 91.3 &
 90.9 &
 86.6
\\
-w F\_ESAM                       &
 \textbf{68.5} &
 \textbf{90.9} &
 \textbf{93.5}&
 \textbf{91.6}&
 \textbf{91.0}&
 \textbf{90.0}
\\ \midrule
-w GSAM                                   &
 67.0 &
 89.7 &
 92.6 &
 91.9 &
 91.7 &
 86.3
 \\
-w F\_GSAM                       &
 \textbf{70.0}&
 \textbf{90.7}&
 \textbf{93.2}&
 \textbf{92.3}&
 \textbf{92.0}&
 \textbf{86.6}
 \\ \bottomrule 
\end{tabular}
}
\end{table}

%% file: figures/training_sizes.tex
\begin{figure*}[ht]
	\centering
	\includegraphics[width=1\textwidth]{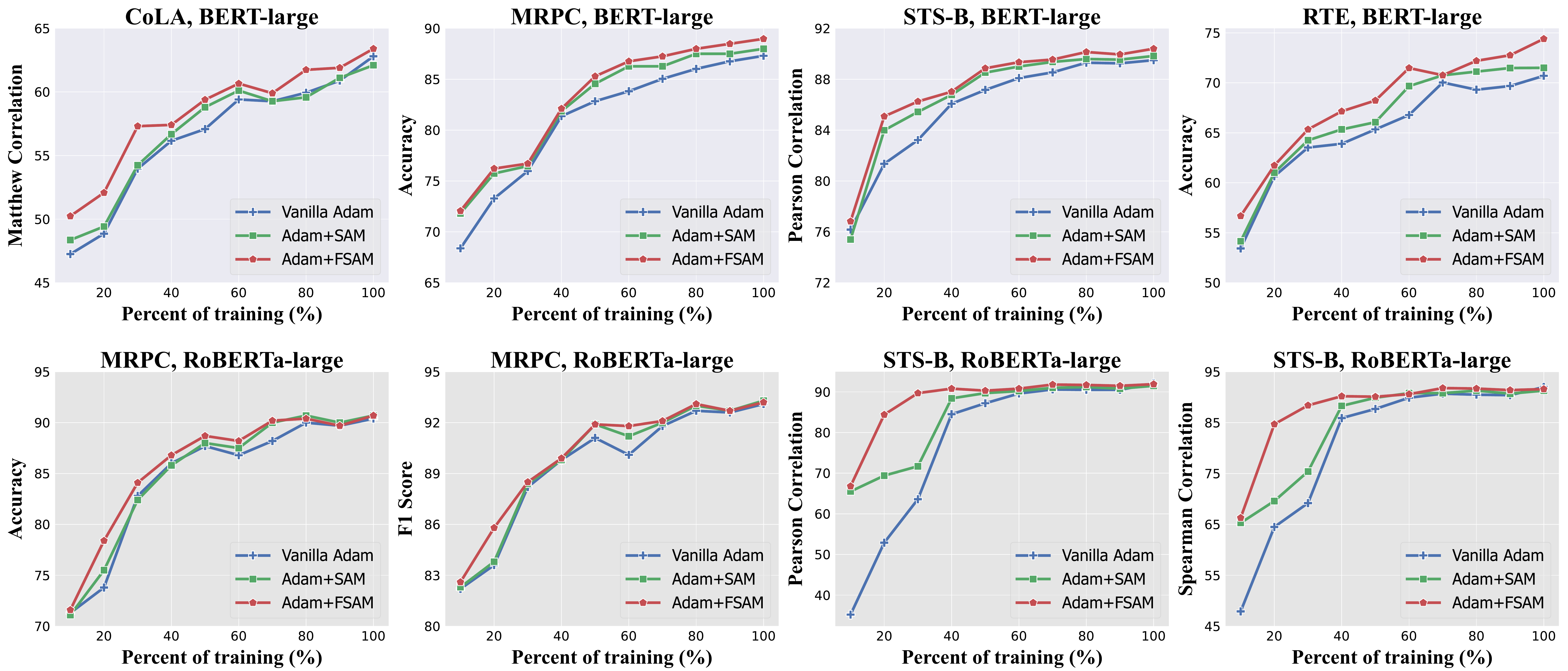} 
	\caption{Results at various training data sampling rates. BERT-large and RoBERTa-large models are used. We can see that our proposed module improves SAM by a large margin across all data size regimes.
	}
	\label{fig:training_size}
\end{figure*}

%% file: figures/task_generalization.tex
\begin{figure}[t]
	\centering
	\includegraphics[width=0.45\textwidth]{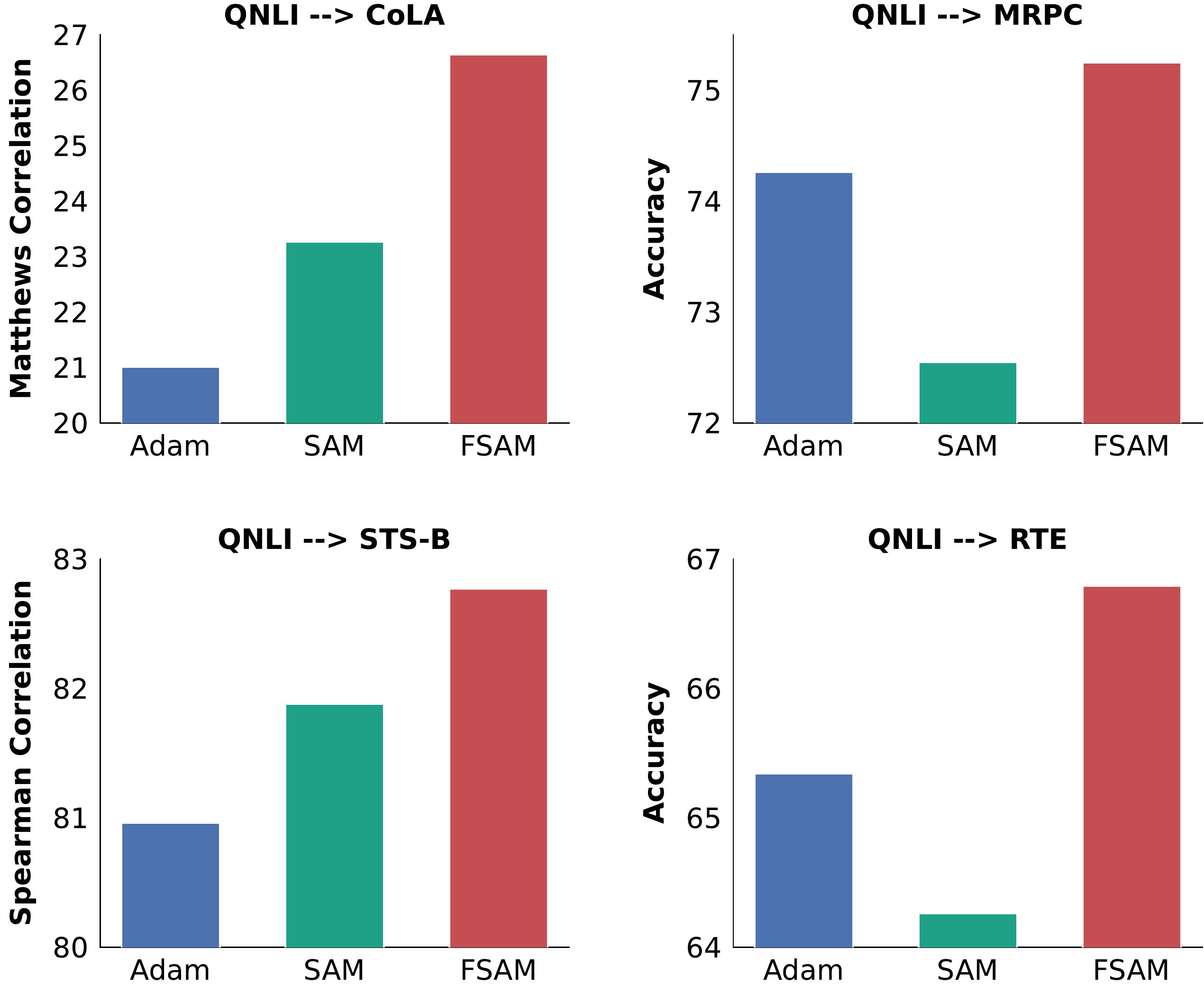} 
	\caption{Analysis of task generalization. The model is fine-tuned on QNLI task and transferred to four different tasks. We can see that our FSAM consistently brings better generalization compared with vanilla SAM.
	}
	\label{fig:task_transfer}
\end{figure}

%% file: figures/3d-loss.tex
\begin{figure*}[h]
	\centering
	\includegraphics[width=0.91\textwidth]{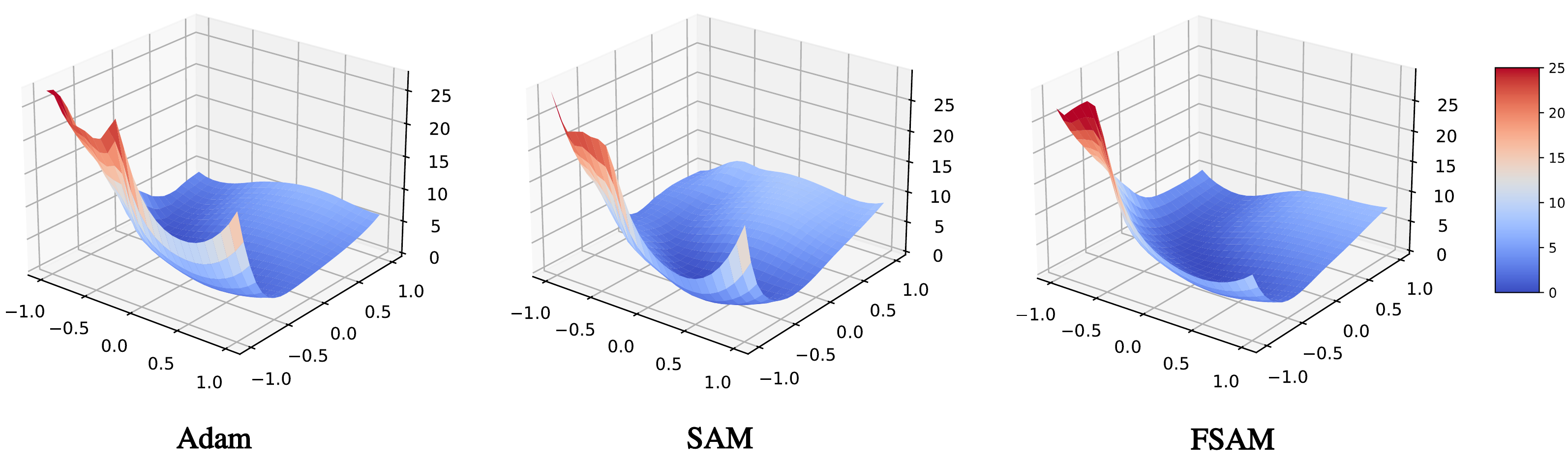} 
	\caption{The loss surface of RoBERTa-base fine-tuned on CoLA with different optimizers. It can be seen that FSAM smooths the loss surface effectively, \textit{i.e.}, improving the model generalization.
	}
	\label{fig:3d_loss}
\end{figure*}

%% file: figures/2d-loss.tex
\begin{figure}[h]
	\centering
	\includegraphics[width=0.48\textwidth]{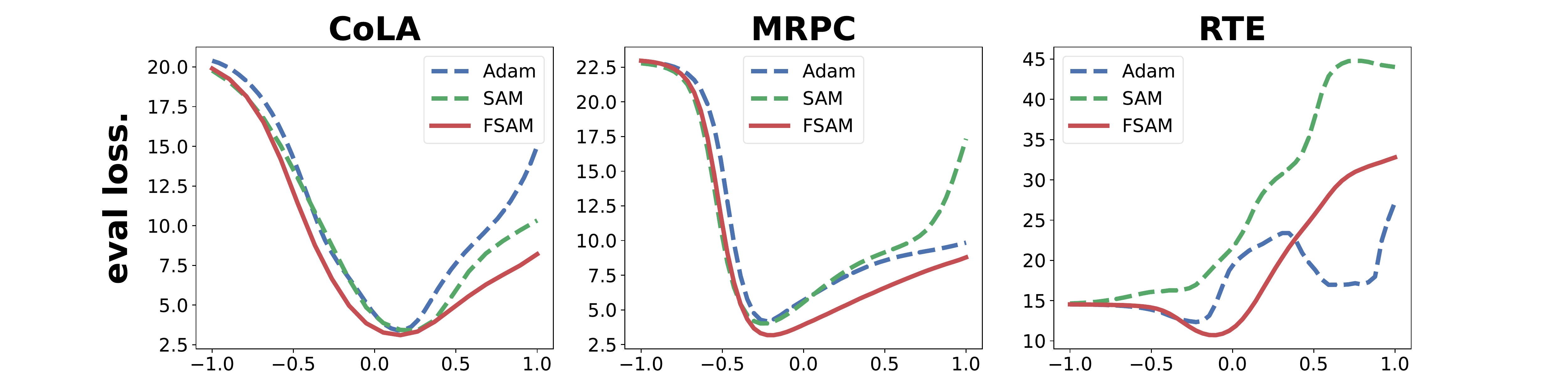} 
	\caption{1D visualization of loss landscapes of RoBERTa-base model fine-tuned on different tasks. 
	}
	\label{fig:1d_loss}
\end{figure}

%% file: section/conclusion.tex
\section{Conclusion}

In this paper, we improve the recently-proposed SAM optimization method with a novel Fisher mask strategy, and propose a new approach FSAM. 
Different from the vanilla SAM that adds a constant perturbation to all parameters, FSAM uses the Fisher information to calculate the Fisher mask and further obtains the sparse perturbation. Such a method can not only reduce the computation cost of optimization potentially, but also make the optimizer focus on the optimization of the important sharper parameters.
Extensive experiments on five PLMs and various language understanding and generation tasks show that our FSAM consistently improves the performance of SAM by a large margin across all PLMs and tasks. Additionally, in-depth analysis and discussion demonstrate the robustness and universality of FSAM on improving the generalization of language models.

%% file: section/limitations.tex
\section*{Limitations}
Indeed, our work has some potential limitations, and we will discuss them in this section. Firstly, we only evaluate the BART on two generation tasks with different optimizers, and prove the effectiveness of our FSAM optimization method. It would be more valuable to consider other sequence-to-sequence PLMs and more generation tasks, \textit{e.g.}, fine-tuning T5~\cite{raffel2020exploring} on CNN-DM~\cite{hermann2015teaching}. 

Additionally, as aforementioned in Section~\ref{sec_intro}, we do not achieve the actual sparse training in this work, due to the limitation of the hardware. Specifically, to actually accelerate the unstructured sparsity (fine-grained sparsity), we need to implement the relevant sparse matrix calculation using the CUDA API on the recent NVIDIA Ampere A100 GPUs~\cite{choquette2021nvidia} equipped with Sparse Tensor Cores~\cite{pool2020accelerating} (Notably, although there is python API (\textit{i.e.}, ASP) provided by NVIDIA for accelerating the unstructured sparsity, it is only applicable to accelerate the model parameter sparsity, but not to the gradient-level sparse acceleration in our FSAM scenario). Unfortunately, it is relatively impracticable for us to do that. However, we still believe that FSAM has great potential to achieve true training acceleration in the future, with the development of hardware for fine-grained sparse operation.

%% file: section/acknowledgement.tex
\section*{Acknowledgements}
We are grateful to the anonymous reviewers and the area chair for their insightful comments and suggestions.
This work was supported in part by the National Natural Science Foundation of China under Grants 62141112, 62076186 and 62225113, and in part by the Science and Technology Major Project of Hubei Province (Next-Generation AI Technologies) under Grant 2019AEA170. The numerical calculations in this paper have been done on the supercomputing system in the Supercomputing Center of Wuhan University.

%% file: section/appendix.tex
\clearpage
\appendix
\section{Appendix}
\label{sec:appendix}

\subsection{Missing Proof}
\label{appendix_proof}
\begin{lemma}
\label{baselemma}
For any two vectors $x,y\in \mathbb{R}^d$, and any scalar $\alpha > 1$, we have the following inequality
\begin{align*}
    \langle x,y \rangle \leq \frac{\alpha^2}{2}||x||^2 + \frac{1}{2\alpha^2}||y||^2
\end{align*}
\end{lemma}
\begin{proof}\let\qed\relax
\begin{align*}
    &RHS=\frac{\alpha^2}{2}\sum_{j=1}^d (x)_j^2 + \frac{1}{2\alpha^2} \sum_{j=1}^d (y)_j^2 \\& \geq \sum_{j=1}^d2\sqrt{\frac{\alpha^2}{2}\cdot (x)_j^2 \cdot \frac{1}{2\alpha^2}(y)_j^2}=LHS
\end{align*}
\end{proof}

\begin{lemma}
\label{lemma1}
We have the following inequality:
\begin{align*}
    &\langle \nabla f(x_t), \frac{\gamma_t}{b}\sum_{u\in B} \nabla f_i(x_t + \rho \frac{\sum \nabla f_i(x_t)\odot m_t}{||\sum \nabla f_i(x_t)||})\odot \\& (\frac{1}{\sqrt{\hat{v_{t-1}}}}-\frac{1}{\sqrt{\hat{v_{t}}}}) \rangle \leq \gamma_t G^2 ||\frac{1}{\sqrt{\hat{v_{t-1}}}} - \frac{1}{\sqrt{\hat{v_{t}}}}||_1
\end{align*}
\end{lemma}

\begin{proof}\let\qed\relax
\begin{align*}
    &\langle \nabla f(x_t), \frac{\gamma_t}{b}\sum_{i\in B} \nabla f_i(x_t + \rho \frac{\sum \nabla f_i(x_t)\odot m_t}{||\sum \nabla f_i(x_t)||}) \\& \odot (\frac{1}{\sqrt{\hat{v_{t-1}}}}-\frac{1}{\sqrt{\hat{v_{t}}}}) \rangle 
    \\
    &\leq \frac{\gamma_t}{b}  \sum_{j=1}^d|(\nabla f(x_t))_{(j)}| \\& \times |\sum_{i\in B}\nabla f_i(x_t+\rho \frac{\sum \nabla f_i(x_t)\odot m_t}{||\sum \nabla f_i(x_t)||}) \\& \odot (\frac{1}{\sqrt{\hat{v_{t-1}}}}-\frac{1}{\sqrt{\hat{v_{t}}}})_{(j)}|
    \\
    &\leq \frac{\gamma_tG}{b} \sum_{j=1}^d \sum_{i\in B} |\nabla f_i(x_t+\rho \frac{\sum \nabla f_i(x_t)\odot m_t}{||\sum \nabla f_i(x_t)||}) \\& \odot (\frac{1}{\sqrt{\hat{v_{t-1}}}}-\frac{1}{\sqrt{\hat{v_{t}}}})_{(j)}|
    \\
    &\leq \frac{\gamma_t G^2}{b} \sum_{j=1}^d \sum_{i\in B} |(\frac{1}{\sqrt{\hat{v_{t-1}}}}-\frac{1}{\sqrt{\hat{v_{t}}}})_{(j)}|
    \\
    &\leq \gamma_t G^2 ||(\frac{1}{\sqrt{\hat{v_{t-1}}}}-\frac{1}{\sqrt{\hat{v_{t}}}})_{(j)}||_1
\end{align*}
\end{proof}

\begin{lemma}
\label{lemma2}
With the term defined before, we have the following inequality
\begin{align*}
&\langle \nabla f(x_t), \frac{\gamma_t}{b}\sum_{i\in B}\nabla f_i(x_t + \rho \frac{\nabla f(x_t) \odot m_t}{||\nabla f(x_t)||}) \\& \odot \frac{1}{\sqrt{\hat{v}_{t-1}}}  - \frac{\gamma_t}{b}\sum_{i\in B}\nabla f_i(x_t \\&+ \rho \frac{\sum \nabla f_i(x_t) \odot m_t}{||\sum \nabla f_i(x_t)||}) \odot \frac{1}{\sqrt{\hat{v}_{t-1}}} \rangle  \\
&\leq \frac{\gamma_t}{2\mu^2}||\nabla f(x_t)\odot \sqrt{\frac{1}{\sqrt{\hat{v}_{t-1}}}}||^2 +  \frac{2\mu^2\gamma_t L^2\rho^2}{\epsilon} ,
\end{align*}
where the $\mu>1$ is a undetermined scalar.
\end{lemma}

\begin{proof}\let\qed\relax
\begin{align*}
    &LHS =\langle \nabla f(x_t)\odot \sqrt{\frac{1}{\sqrt{\hat{v}_{t-1}}}},\\& \frac{\gamma_t}{b}\sum_{i\in B}(\nabla f_i(x_t + \rho \frac{\nabla f(x_t) \odot m_t}{||\nabla f(x_t)||}) \\&- \nabla f_i(x_t + \rho \frac{\sum \nabla f_i(x_t) \odot m_t}{||\sum \nabla f_i(x_t)||})) \odot \sqrt{\frac{1}{\sqrt{\hat{v}_{t-1}}}} \rangle 
\end{align*}
By using Lemma~\ref{baselemma}, we have
\begin{align*}
    &LHS \leq \frac{\mu^2 \gamma_t}{2b^2} ||\sum_{i\in B}(\nabla f_i(x_t + \rho \frac{\nabla f(x_t) \odot m_t}{||\nabla f(x_t)||}) \\&- \nabla f_i(x_t + \rho \frac{\sum \nabla f_i(x_t) \odot m_t}{||\sum \nabla f_i(x_t)||})) \odot \sqrt{\frac{1}{\sqrt{\hat{v}_{t-1}}}}||^2 \\& + \frac{\gamma_t}{2\mu^2}||\nabla f(x_t)\odot \sqrt{\frac{1}{\sqrt{\hat{v}_{t-1}}}}||^2
    \\
    &\leq \frac{\mu^2 \gamma_t}{2b} \sum ||(\nabla f_i(x_t + \rho \frac{\nabla f(x_t) \odot m_t}{||\nabla f(x_t)||}) \\&- \nabla f_i(x_t + \rho \frac{\sum \nabla f_i(x_t) \odot m_t}{||\sum \nabla f_i(x_t)||})) \odot\sqrt{\frac{1}{\sqrt{\hat{v}_{t-1}}}}||^2 \\& + \frac{\gamma_t}{2\mu^2}||\nabla f(x_t)\odot \sqrt{\frac{1}{\sqrt{\hat{v}_{t-1}}}}||^2
    \\
    &\leq \frac{\mu^2 \gamma_t}{2b} \sum ||\nabla f_i(x_t + \rho \frac{\nabla f(x_t) \odot m_t}{||\nabla f(x_t)||}) \\&- \nabla f_i(x_t + \rho \frac{\sum \nabla f_i(x_t) \odot m_t}{||\sum \nabla f_i(x_t)||}) ||^2 \\& \cdot ||\sqrt{\frac{1}{\sqrt{\hat{v}_{t-1}}}}||^2_\infty  + \frac{\gamma_t}{2\mu^2}||\nabla f(x_t)\odot \sqrt{\frac{1}{\sqrt{\hat{v}_{t-1}}}}||^2
    \\
    &\leq \frac{\gamma_t }{2\mu^2}||\nabla f(x_t)\odot \sqrt{\frac{1}{\sqrt{\hat{v}_{t-1}}}}||^2 +  \frac{\mu^2\gamma_tL^2\rho^2}{2b\delta} \cdot \\& \sum ||\frac{\nabla f(x_t) \odot m_t}{||\nabla f(x_t)||} - \frac{\sum \nabla f_i(x_t) \odot m_t}{||\sum \nabla f_i(x_t)||}||^2 
    \\
    &\leq \frac{\gamma_t}{2\mu^2}||\nabla f(x_t)\odot \sqrt{\frac{1}{\sqrt{\hat{v}_{t-1}}}}||^2 +  \frac{\mu^2\gamma_tL^2\rho^2}{2b\delta} \cdot \\& \sum ||(\frac{\nabla f(x_t)}{||\nabla f(x_t)||} - \frac{\sum \nabla f_i(x_t) }{||\sum \nabla f_i(x_t)||} )\odot m_t||^2 
    \\
    &\leq \frac{\gamma_t}{2\mu^2}||\nabla f(x_t)\odot \sqrt{\frac{1}{\sqrt{\hat{v}_{t-1}}}}||^2 +  \frac{\mu^2\gamma_tL^2\rho^2}{2b\delta} \cdot \\& \sum ||\frac{\nabla f(x_t)}{||\nabla f(x_t)||} - \frac{\sum \nabla f_i(x_t) }{||\sum \nabla f_i(x_t)||}||^2  \cdot ||m_t||^2_\infty
    \\
    &\leq \frac{\gamma_t}{2\mu^2}||\nabla f(x_t)\odot \sqrt{\frac{1}{\sqrt{\hat{v}_{t-1}}}}||^2 +  \frac{\mu^2\gamma_tL^2\rho^2}{2b\delta} \cdot \\& \sum ||\frac{\nabla f(x_t)}{||\nabla f(x_t)||} - \frac{\sum \nabla f_i(x_t) }{||\sum \nabla f_i(x_t)||}||^2 
    \\
    &\leq \frac{\gamma_t}{2\mu^2}||\nabla f(x_t)\odot \sqrt{\frac{1}{\sqrt{\hat{v}_{t-1}}}}||^2 +  \frac{2\mu^2\gamma_tL^2\rho^2}{\delta} 
\end{align*}
The $\mu>0$ is the term to be determined.
\end{proof}

\begin{lemma}
\label{lemma3}
We have the following inequality:
\begin{align*}
    &\mathbb{E}\langle \nabla f(x_t), -\frac{\gamma_t}{b}\sum_{i\in B}\nabla f_i(x_t + \rho \frac{\nabla f(x_t) \odot m_t}{||\nabla f(x_t)||})\\& \odot \frac{1}{\sqrt{\hat{v}_{t-1}}} \rangle \leq -\gamma_t ||\nabla f(x_t)\odot \sqrt{\frac{1}{\sqrt{\hat{v}_{t-1}}}}||^2 \\
    &+ \mathbb{E}\frac{\gamma_t}{2\beta^2}||\nabla f(x_t)\odot \sqrt{\frac{1}{\sqrt{\hat{v}_{t-1}}}}||^2 
    + \frac{\gamma_t\beta^2L^2\rho^2}{2\delta},
\end{align*}
where the $\beta>1$ is a undetermined scalar.
\end{lemma}

\begin{proof}\let\qed\relax
\begin{align*}
    LHS =& -\gamma_t ||\nabla f(x_t)\odot \sqrt{\frac{1}{\sqrt{\hat{v}_{t-1}}}}||^2 \\
    +& \mathbb{E}\langle \nabla f(x_t), \frac{\gamma_t}{b}\sum_{i\in B}(\nabla f(x_t) \\-& \nabla f_i(x_t + \rho \frac{\nabla f(x_t) \odot m_t}{||\nabla f(x_t)||})) \odot \frac{1}{\sqrt{\hat{v}_{t-1}}} \rangle
    \\
    =& -\gamma_t ||\nabla f(x_t)\odot \sqrt{\frac{1}{\sqrt{\hat{v}_{t-1}}}}||^2 \\
    +& \mathbb{E}\langle \nabla f(x_t), \frac{\gamma_t}{b}\sum_{i\in B}(\nabla f_i(x_t) \\-& \nabla f_i(x_t + \rho \frac{\nabla f(x_t) \odot m_t}{||\nabla f(x_t)||})) \odot \frac{1}{\sqrt{\hat{v}_{t-1}}} \rangle
    \\
\end{align*}
By using the Lemma~\ref{baselemma}, we have
\begin{align*}
    LHS \leq& -\gamma_t ||\nabla f(x_t)\odot \sqrt{\frac{1}{\sqrt{\hat{v}_{t-1}}}}||^2 \\
    +& \mathbb{E}\frac{\gamma_t}{2\beta^2}||\nabla f(x_t)\odot \sqrt{\frac{1}{\sqrt{\hat{v}_{t-1}}}}||^2 \\ +& \mathbb{E} \frac{\gamma_t \beta^2}{2}|| \frac{1}{b}\sum_{i\in B}(\nabla f_i(x_t) \\-& \nabla f_i(x_t + \rho \frac{\nabla f(x_t) \odot m_t}{||\nabla f(x_t)||})) \odot \sqrt{\frac{1}{\sqrt{\hat{v}_{t-1}}}}||^2
    \\
    \leq& -\gamma_t ||\nabla f(x_t)\odot \sqrt{\frac{1}{\sqrt{\hat{v}_{t-1}}}}||^2 \\
    +& \mathbb{E}\frac{\gamma_t}{2\beta^2}||\nabla f(x_t)\odot \sqrt{\frac{1}{\sqrt{\hat{v}_{t-1}}}}||^2 \\ +& \mathbb{E} \frac{\gamma_t \beta^2}{2\delta}|| \frac{1}{b}\sum_{i\in B}(\nabla f_i(x_t) \\-& \nabla f_i(x_t + \rho \frac{\nabla f(x_t) \odot m_t}{||\nabla f(x_t)||})) ||^2
    \\
    \leq& -\gamma_t ||\nabla f(x_t)\odot \sqrt{\frac{1}{\sqrt{\hat{v}_{t-1}}}}||^2 \\
    +& \mathbb{E}\frac{\gamma_t}{2\beta^2}||\nabla f(x_t)\odot \sqrt{\frac{1}{\sqrt{\hat{v}_{t-1}}}}||^2 \\ +& \mathbb{E} \frac{\gamma_t \beta^2}{2b\delta} \sum_{i\in B}||(\nabla f_i(x_t) \\-& \nabla f_i(x_t + \rho \frac{\nabla f(x_t) \odot m_t}{||\nabla f(x_t)||})) ||^2
    \\
\end{align*}
By the definition of L-smooth, we have
\begin{align*}
    LHS\leq& -\gamma_t ||\nabla f(x_t)\odot \sqrt{\frac{1}{\sqrt{\hat{v}_{t-1}}}}||^2 \\
    +& \mathbb{E}\frac{\gamma_t}{2\beta^2}||\nabla f(x_t)\odot \sqrt{\frac{1}{\sqrt{\hat{v}_{t-1}}}}||^2 \\
    +& \frac{\gamma_t\beta^2L^2\rho^2}{2b\delta}\mathbb{E}\sum_{i\in B}||\frac{\nabla f(x_t)}{||\nabla f(x_t)||}\odot m_t||^2
    \\
    \leq& -\gamma_t ||\nabla f(x_t)\odot \sqrt{\frac{1}{\sqrt{\hat{v}_{t-1}}}}||^2 \\
    +& \mathbb{E}\frac{\gamma_t}{2\beta^2}||\nabla f(x_t)\odot \sqrt{\frac{1}{\sqrt{\hat{v}_{t-1}}}}||^2 \\
    +& \frac{\gamma_t\beta^2L^2\rho^2}{2\delta}
\end{align*}
\end{proof}

\begin{lemma}
\label{lemma4}
We have the following inequality:
\begin{align*}
    &\frac{L}{2}\mathbb{E}||x_{t+1}-x_t||^2 \leq \frac{\gamma_t^2L}{2}[3\frac{1+\alpha}{\alpha\delta}\\&\cdot(\mathbb{E}||\nabla f(x_t)\odot \sqrt{\frac{1}{\sqrt{\hat{v}_{t-1}}}}||^2 + \frac{L\rho^2}{\delta}) \\& +\frac{\sigma^2}{b\delta} + (1+\alpha)G^2\mathbb{E}||(\frac{1}{\sqrt{\hat{v}_{t}}} - \frac{1}{\sqrt{\hat{v}_{t-1}}})||^2]
\end{align*}
\end{lemma}
\begin{proof}\let\qed\relax
\begin{align*}
    &\frac{L}{2}\mathbb{E}||x_{t+1}-x_t||^2 
    \\& =\frac{L}{2}\mathbb{E}||\frac{\gamma_t}{b}\sum(\nabla f_i(x_t \\& + \frac{\rho}{b}\frac{\sum\nabla f_i(x_t)\odot m_t}{||\sum\nabla f_i(x_t)||})) \odot \frac{1}{\sqrt{\hat{v}_{t}}}||^2
    \\& =\frac{L\gamma_t^2}{2}\mathbb{E}||\frac{1}{b}\sum(\nabla f_i(x_t + \frac{\rho}{b}\frac{\sum\nabla f_i(x_t)\odot m_t}{||\sum\nabla f_i(x_t)||})\\& \odot \frac{1}{\sqrt{\hat{v}_{t-1}}})  + \frac{1}{b}\sum(\nabla f_i(x_t + \frac{\rho}{b}\frac{\sum\nabla f_i(x_t)\odot m_t}{||\sum\nabla f_i(x_t)||}))\\&\odot (\frac{1}{\sqrt{\hat{v}_{t}}} - \frac{1}{\sqrt{\hat{v}_{t-1}}})||^2
\end{align*}
By using the Lemma~\ref{baselemma}, we can obtain
\begin{align*}
    &\frac{L}{2}\mathbb{E}||x_{t+1}-x_t||^2 
    \\& \leq \frac{L\gamma_t^2}{2}[(1+
    \frac{1}{\alpha})\mathbb{E}||\frac{1}{b}\sum(\nabla f_i(x_t + \\& \frac{\rho}{b}\frac{\sum\nabla f_i(x_t)\odot m_t}{||\sum\nabla f_i(x_t)||}) \odot \frac{1}{\sqrt{\hat{v}_{t-1}}})||^2  \\&+ (1+\alpha)\mathbb{E}|| \frac{1}{b}\sum(\nabla f_i(x_t + \\& \frac{\rho}{b}\frac{\sum\nabla f_i(x_t)\odot m_t}{||\sum\nabla f_i(x_t)||}))\odot (\frac{1}{\sqrt{\hat{v}_{t}}} - \frac{1}{\sqrt{\hat{v}_{t-1}}})||^2]
    \\& \leq \frac{L\gamma_t^2}{2}[(1+\frac{1}{\alpha})\mathbb{E}||\frac{1}{b}\sum(\nabla f_i(x_t + \\& \frac{\rho}{b}\frac{\sum\nabla f_i(x_t)\odot m_t}{||\sum\nabla f_i(x_t)||}) \odot \frac{1}{\sqrt{\hat{v}_{t-1}}})||^2 \\& + (1+\alpha)G^2\mathbb{E}||(\frac{1}{\sqrt{\hat{v}_{t}}} - \frac{1}{\sqrt{\hat{v}_{t-1}}})||^2]
\end{align*}
Similar to the proof of Lemma~\ref{lemma2}, we have
\begin{align*}
    &\frac{L}{2}\mathbb{E}||x_{t+1}-x_t||^2 
    \\& \leq \frac{L\gamma_t^2}{2}[(1+\frac{1}{\alpha})\mathbb{E}||\frac{1}{b}\sum(\nabla f_i(x_t \\&+ \frac{\rho}{b}\frac{\sum\nabla f_i(x_t)\odot m_t}{||\sum\nabla f_i(x_t)||})\odot \sqrt{\frac{1}{\sqrt{\hat{v}_{t-1}}}})||^2 \\& \cdot ||\sqrt{\frac{1}{\sqrt{\hat{v}_{t-1}}}}||^2_\infty \\& + (1+\alpha)G^2\mathbb{E}||(\frac{1}{\sqrt{\hat{v}_{t}}} - \frac{1}{\sqrt{\hat{v}_{t-1}}})||^2]
    \\ & \leq \frac{L\gamma_t^2}{2}[(\frac{1+\alpha}{\alpha\delta})\mathbb{E}||\frac{1}{b}\sum(\nabla f_i(x_t \\&+ \frac{\rho}{b}\frac{\sum\nabla f_i(x_t)\odot m_t}{||\sum\nabla f_i(x_t)||})\odot \sqrt{\frac{1}{\sqrt{\hat{v}_{t-1}}}})||^2 \\& + (1+\alpha)G^2\mathbb{E}||(\frac{1}{\sqrt{\hat{v}_{t}}} - \frac{1}{\sqrt{\hat{v}_{t-1}}})||^2]
\end{align*}
By splitting the term $\mathbb{E}||\frac{1}{b}\sum(\nabla f_i(x_t \\+ \frac{\rho}{b}\frac{\sum\nabla f_i(x_t)\odot m_t}{||\sum\nabla f_i(x_t)||})\odot \sqrt{\frac{1}{\sqrt{\hat{v}_{t-1}}}})$, we have
\begin{align*}
    &LHS \leq  \frac{L\gamma_t^2}{2}[3(\frac{1+\alpha}{\alpha\delta}) \mathbb{E}||\nabla f(x_t)\odot \sqrt{\frac{1}{\sqrt{\hat{v}_{t-1}}}} ||^2 \\& + ||(\frac{1}{b}\sum \nabla f_i(x_t)-\nabla f(x_t))\odot \sqrt{\frac{1}{\sqrt{\hat{v}_{t-1}}}}||^2 \\& + ||(\frac{1}{b}\sum (\nabla f_i(x_t + \rho \frac{\sum \nabla f_i(x_t)}{||\sum \nabla f_i(x_t)||})-\nabla f_i(x_t)) \\& \odot \sqrt{\frac{1}{\sqrt{\hat{v}_{t-1}}}}||^2 + (1+\alpha)G^2\mathbb{E}||(\frac{1}{\sqrt{\hat{v}_{t}}} - \frac{1}{\sqrt{\hat{v}_{t-1}}})||^2]
    \\&
    \leq \frac{L\gamma_t^2}{2}[3(\frac{1+\alpha}{\alpha\delta}) \mathbb{E}||\nabla f(x_t)\odot \sqrt{\frac{1}{\sqrt{\hat{v}_{t-1}}}} ||^2 + \frac{\sigma^2}{b\delta} \\& + ||(\frac{1}{b}\sum (\nabla f_i(x_t + \rho \frac{\sum \nabla f_i(x_t)\odot m_t}{||\sum \nabla f_i(x_t)||})-\nabla f_i(x_t)) \\& \odot \sqrt{\frac{1}{\sqrt{\hat{v}_{t-1}}}}||^2 + (1+\alpha)G^2\mathbb{E}||(\frac{1}{\sqrt{\hat{v}_{t}}} - \frac{1}{\sqrt{\hat{v}_{t-1}}})||^2]
    \\&
    \leq \frac{L\gamma_t^2}{2}[3(\frac{1+\alpha}{\alpha\delta}) \mathbb{E}||\nabla f(x_t)\odot \sqrt{\frac{1}{\sqrt{\hat{v}_{t-1}}}} ||^2 + \frac{\sigma^2}{b\delta} \\& + \frac{1}{\delta}||(\frac{1}{b}\sum (\nabla f_i(x_t + \rho \frac{\sum \nabla f_i(x_t)\odot m_t}{||\sum \nabla f_i(x_t)||})\\& -\nabla f_i(x_t))||^2 + (1+\alpha)G^2\mathbb{E}||(\frac{1}{\sqrt{\hat{v}_{t}}} - \frac{1}{\sqrt{\hat{v}_{t-1}}})||^2]
    \\&
    \leq \frac{\gamma_t^2L}{2}[3\frac{1+\alpha}{\alpha\delta}(\mathbb{E}||\nabla f(x_t)\odot \sqrt{\frac{1}{\sqrt{\hat{v}_{t-1}}}}||^2 + \frac{L\rho^2}{\delta}) \\& +\frac{\sigma^2}{b\delta} + (1+\alpha)G^2\mathbb{E}||(\frac{1}{\sqrt{\hat{v}_{t}}} - \frac{1}{\sqrt{\hat{v}_{t-1}}})||^2]
\end{align*}
\end{proof}

\begin{theorem}
By using the definition of L-smooth and the Lemma~\ref{lemma1}, \ref{lemma2}, \ref{lemma3} and \ref{lemma4}, we have
\begin{align*}
    \frac{1}{T}\sum_{t=0}^{T-1}&\mathbb{E}||\nabla f(x_t)||^2 
    \\&\leq \frac{2G  f(x_{0})-f^*}{\gamma_tT} \\& + \frac{20GL^2\rho^2}{\delta} + \frac{2G^3}{T}d(\frac{1}{\delta}-\frac{1}{G}) \\& + \frac{4G\gamma_tL}{\delta}\frac{L\rho^2}{\delta}+ \frac{4G\gamma_tL}{\delta} \frac{\sigma^2}{b\delta} \\&  + \frac{4\gamma_tLG^3}{T}d(G^2-\delta^2)
\end{align*}
\end{theorem}

\begin{proof}\let\qed\relax
\begin{align*}
    f(x_{t+1}) \leq& f(x_t) + \langle \nabla f(x_t), x_{t+1} - x_t \rangle 
    \\
    +& \frac{L}{2} ||x_{t+1}-x_t||^2
\end{align*}
By re-range it, we obtain
\begin{align*}
    &f(x_{t+1}) - f(x_t)
    \\
    \leq& \langle \nabla f(x_t), x_{t+1} - x_t \rangle + \frac{L}{2} ||x_{t+1}-x_t||^2
    \\
    =& \langle \nabla f(x_t), - \frac{\gamma_t}{b}\sum_{i\in B}\nabla f_i(x_t \\+ & \rho \frac{\sum \nabla f_i(x_t) \odot m_t}{||\sum \nabla f_i(x_t)||}) \odot  \frac{1}{\sqrt{\hat{v}_t}}\rangle + \frac{L}{2} ||x_{t+1}-x_t||^2
    \\
    =& \langle \nabla f(x_t), - \frac{\gamma_t}{b}\sum_{i\in B}\nabla f_i(x_t \\+ & \rho \frac{\sum \nabla f_i(x_t) \odot m_t}{||\sum \nabla f_i(x_t)||}) \odot \frac{1}{\sqrt{\hat{v}_{t-1}}}\rangle \\
    +& \langle \nabla f(x_t), \frac{\gamma_t}{b}\sum_{i\in B}\nabla f_i(x_t +  \rho \frac{\sum \nabla f_i(x_t) \odot m_t}{||\sum \nabla f_i(x_t)||})  \\ \odot & (\frac{1}{\sqrt{\hat{v}_{t-1}}} - \frac{1}{\sqrt{\hat{v}_{t}}}) \rangle  + \frac{L}{2} ||x_{t+1}-x_t||^2
    \\
    =& \langle \nabla f(x_t), -\frac{\gamma_t}{b}\sum_{i\in B}\nabla f_i(x_t +  \rho \frac{\nabla f(x_t) \odot m_t}{||\nabla f(x_t)||})\\\odot &  \frac{1}{\sqrt{\hat{v}_{t-1}}} \rangle + \langle \nabla f(x_t), \frac{\gamma_t}{b}\sum_{i\in B}\nabla f_i(x_t + \\\rho &  \frac{\nabla f(x_t) \odot m_t}{||\nabla f(x_t)||}) \odot \frac{1}{\sqrt{\hat{v}_{t-1}}} - \frac{\gamma_t}{b} \\\cdot &  \sum_{i\in B}\nabla f_i(x_t + \rho \frac{\sum \nabla f_i(x_t) \odot m_t}{||\sum \nabla f_i(x_t)||})  \odot \frac{1}{\sqrt{\hat{v}_{t-1}}} \rangle
    \\
    +& \langle \nabla f(x_t), \frac{\gamma_t}{b}\sum_{i\in B}\nabla f_i(x_t + \rho \frac{\sum \nabla f_i(x_t) \odot m_t}{||\sum \nabla f_i(x_t)||})  \\ \odot &  (\frac{1}{\sqrt{\hat{v}_{t-1}}} - \frac{1}{\sqrt{\hat{v}_{t}}}) \rangle  + \frac{L}{2} ||x_{t+1}-x_t||^2
\end{align*}
From the Lemma~\ref{lemma1}, we have the following inequality. Specifically,
\begin{align*}
    &f(x_{t+1}) - f(x_t) \\ 
    &\leq \langle \nabla f(x_t), -\frac{\gamma_t}{b}\sum_{i\in B}\nabla f_i(x_t + \rho \frac{\nabla f(x_t) \odot m_t}{||\nabla f(x_t)||}) \\& \odot \frac{1}{\sqrt{\hat{v}_{t-1}}} \rangle \\& + \langle \nabla f(x_t), \frac{\gamma_t}{b}\sum_{i\in B}\nabla f_i(x_t + \rho \frac{\nabla f(x_t) \odot m_t}{||\nabla f(x_t)||}) \\& \odot \frac{1}{\sqrt{\hat{v}_{t-1}}} \\& - \frac{\gamma_t}{b}\sum_{i\in B}\nabla f_i(x_t + \rho \frac{\sum \nabla f_i(x_t) \odot m_t}{||\sum \nabla f_i(x_t)||}) \odot \frac{1}{\sqrt{\hat{v}_{t-1}}} \rangle
    \\
    &+ \gamma_t G^2 ||\frac{1}{\sqrt{\hat{v_{t-1}}}} - \frac{1}{\sqrt{\hat{v_{t}}}}||_1 + \frac{L}{2} ||x_{t+1}-x_t||^2
\end{align*}
From the Lemma~\ref{lemma2}, we have the following inequality. Specifically,
\begin{align*}
    f&(x_{t+1}) - f(x_t) \\ 
    \leq& \langle \nabla f(x_t), -\frac{\gamma_t}{b}\sum_{i\in B}\nabla f_i(x_t + \rho \frac{\nabla f(x_t) \odot m_t}{||\nabla f(x_t)||}) \\\odot & \frac{1}{\sqrt{\hat{v}_{t-1}}} \rangle + \frac{L}{2}||x_{t+1}-x_t||^2  \\+ & \frac{\gamma_t}{2\mu^2}||\nabla f(x_t)\odot \sqrt{\frac{1}{\sqrt{\hat{v}_{t-1}}}}||^2 + \frac{2\mu^2\gamma_tL^2\rho^2}{\delta}
    \\
    +& \gamma_t G^2 ||\frac{1}{\sqrt{\hat{v_{t-1}}}} - \frac{1}{\sqrt{\hat{v_{t}}}}||_1 + \frac{L}{2} ||x_{t+1}-x_t||^2
\end{align*}
By taking the expectation, we have the following inequality. Specifically,
\begin{align*}
    \mathbb{E}&f(x_{t+1}) - f(x_t)
    \\
    \leq& \mathbb{E}\langle \nabla f(x_t), -\frac{\gamma_t}{b}\sum_{i\in B}\nabla f_i(x_t + \rho \frac{\nabla f(x_t) \odot m_t}{||\nabla f(x_t)||})\\\odot& \frac{1}{\sqrt{\hat{v}_{t-1}}} \rangle + \frac{L}{2}\mathbb{E}||x_{t+1}-x_t||^2\\+& \frac{\gamma_t}{2\mu^2}||\nabla f(x_t)\odot \sqrt{\frac{1}{\sqrt{\hat{v}_{t-1}}}}||^2 + \frac{2\mu^2\gamma_tL^2\rho^2}{\delta}
    \\ +& \gamma_t G^2 ||\frac{1}{\sqrt{\hat{v}_{t-1}}} - \frac{1}{\sqrt{\hat{v}_t}}||_1
\end{align*}
From the Lemma~\ref{lemma3}, we have the following inequality. Specifically,
\begin{align*}
    \mathbb{E}&f(x_{t+1}) - f(x_t)
    \\
    \leq& -\gamma_t ||\nabla f(x_t)\odot \sqrt{\frac{1}{\sqrt{\hat{v}_{t-1}}}}||^2 \\
    +& \mathbb{E}\frac{\gamma_t}{2\beta^2}||\nabla f(x_t)\odot \sqrt{\frac{1}{\sqrt{\hat{v}_{t-1}}}}||^2 
    + \frac{\gamma_t\beta^2L^2\rho^2}{2\delta}  \\ +&\frac{L}{2}\mathbb{E}||x_{t+1}-x_t||^2 + \frac{2\mu^2\gamma_tL^2\rho^2}{\delta}\\ +& \frac{\gamma_t}{2\mu^2}||\nabla f(x_t)\odot \sqrt{\frac{1}{\sqrt{\hat{v}_{t-1}}}}||^2 
    \\  +& \gamma_t G^2 ||\frac{1}{\sqrt{\hat{v}_{t-1}}} - \frac{1}{\sqrt{\hat{v}_t}}||_1
\end{align*}
From the Lemma~\ref{lemma4}, we can obtain
\begin{align*}
    \mathbb{E}&f(x_{t+1}) - f(x_t)
    \\
    \leq& -\gamma_t ||\nabla f(x_t)\odot \sqrt{\frac{1}{\sqrt{\hat{v}_{t-1}}}}||^2 \\
    +& \mathbb{E}\frac{\gamma_t}{2\beta^2}||\nabla f(x_t)\odot \sqrt{\frac{1}{\sqrt{\hat{v}_{t-1}}}}||^2 
    + \frac{\gamma_t\beta^2L^2\rho^2}{2\delta}  \\ +& \frac{2\mu^2\gamma_tL^2\rho^2}{\delta} + \frac{\gamma_t}{2\mu^2}||\nabla f(x_t)\odot \sqrt{\frac{1}{\sqrt{\hat{v}_{t-1}}}}||^2 
    \\  +& \gamma_t G^2 ||\frac{1}{\sqrt{\hat{v}_{t-1}}} - \frac{1}{\sqrt{\hat{v}_t}}||_1 \\ +& \frac{\gamma_t^2L}{2}[3\frac{1+\alpha}{\alpha\delta}(\mathbb{E}||\nabla f(x_t)\odot \sqrt{\frac{1}{\sqrt{\hat{v}_{t-1}}}}||^2 + \frac{L\rho^2}{\delta}) \\ +&\frac{\sigma^2}{b\delta} + (1+\alpha)G^2\mathbb{E}||(\frac{1}{\sqrt{\hat{v}_{t}}} - \frac{1}{\sqrt{\hat{v}_{t-1}}})||^2]
    \\ =& -\gamma_t(1-\frac{1}{2\mu^2}-\frac{1}{2\beta^2}-\frac{3\mu L(1+\alpha)}{2\alpha^2})\\ \cdot& \mathbb{E}||\nabla f(x_t)\odot \sqrt{\frac{1}{\sqrt{\hat{v}_{t-1}}}}||^2 \\+ & \frac{2\mu^2\gamma_tL^2\rho^2}{\delta} + \gamma_tG^2\mathbb{E}||\frac{1}{\sqrt{\hat{v}_{t}}} - \frac{1}{\sqrt{\hat{v}_{t-1}}}||_1 \\+ & \frac{\gamma_t \beta^2 L^2 \rho^2}{2\delta}  + \frac{3\gamma_t^2L(1+\alpha)}{2\alpha\delta}(\frac{L\rho^2}{\delta} + \frac{\sigma^2}{b\delta}) \\+ & \frac{\gamma_t^2L(1+\alpha)G^2}{2}\mathbb{E}||(\frac{1}{\sqrt{\hat{v}_{t}}} - \frac{1}{\sqrt{\hat{v}_{t-1}}})||^2
\end{align*}
Set the $\mu^2=\beta^2=4, \alpha=3$ and set the $\frac{\gamma_tL}{\delta}\leq \frac{1}{8}$, we can simplify the inequality
\begin{align*}
    &\mathbb{E}f(x_{t+1}) - f(x_t)
    \\& \leq -\frac{\gamma_t}{2}\mathbb{E}||\nabla f(x_t)\odot \sqrt{\frac{1}{\sqrt{\hat{v}_{t-1}}}}||^2 + \frac{8\gamma_tL^2\rho^2}{\delta} \\& + \gamma_t G^2\mathbb{E}||(\frac{1}{\sqrt{\hat{v}_{t-1}}} - \frac{1}{\sqrt{\hat{v}_{t}}})||_1 \\& +\frac{2\gamma_tL^2\rho^2}{\delta} +
    \frac{2\gamma_t^2L}{\delta}(\frac{L\rho^2}{\delta}+\frac{\sigma^2}{b\delta}) \\& + 2\gamma_t^2LG^2\mathbb{E}||(\frac{1}{\sqrt{\hat{v}_{t}}} - \frac{1}{\sqrt{\hat{v}_{t-1}}})||^2
\end{align*}
Note that the $\frac{1}{\sqrt{\hat{v}_t}}$ is bounded, we re-arrange the inequality and achieve
\begin{align*}
    &\frac{\gamma_t}{2G}\mathbb{E}||\nabla f(x_t)||^2 \leq \frac{\gamma_t}{2} \mathbb{E}||\nabla f(x_t)\odot \sqrt{\frac{1}{\sqrt{\hat{v}_{t-1}}}}||^2
    \\ & \leq -\mathbb{E} f(x_{t+1})+\mathbb{E}f(x_t) + \frac{8\gamma_tL^2\rho^2}{\delta} \\& + \gamma_t G^2\mathbb{E}||(\frac{1}{\sqrt{\hat{v}_{t-1}}} - \frac{1}{\sqrt{\hat{v}_{t}}})||_1 \\& +\frac{2\gamma_tL^2\rho^2}{\delta} + \frac{2\gamma_t^2L}{\delta}(\frac{L\rho^2}{\delta}+\frac{\sigma^2}{b\delta})
    \\& + 2\gamma_t^2LG^2\mathbb{E}||(\frac{1}{\sqrt{\hat{v}_{t}}} - \frac{1}{\sqrt{\hat{v}_{t-1}}})||^2
\end{align*}
We sum up it and we obtain
\begin{align*}
    &\frac{1}{T}\sum_{t=0}^{T-1}\mathbb{E}||\nabla f(x_t)||^2 \leq 2G \frac{\mathbb{E} f(x_{0})-\mathbb{E}f(x_{t+1})}{\gamma_tT} \\ &+ \frac{16GL^2\rho^2}{\delta} + \frac{2G^3}{T} \mathbb{E}\sum_{t=0}^{T-1}||\frac{1}{\sqrt{\hat{v_{t-1}}}} - \frac{1}{\sqrt{\hat{v_{t}}}}||_1
    \\ &+\frac{4GL^2\rho^2}{\delta} + \frac{4G\gamma_tL}{\delta}(\frac{L\rho^2}{\delta}+\frac{\sigma^2}{b\delta}) \\&+\frac{4\gamma_tLG^3}{T}\mathbb{E}\sum_{t=0}^{T-1}||\frac{1}{\sqrt{\hat{v_{t}}}} - \frac{1}{\sqrt{\hat{v_{t-1}}}}||^2
    \\ &\leq  \frac{2G  f(x_{0})-f^*}{\gamma_tT} +  \frac{16GL^2\rho^2}{\delta} \\& + \frac{2G^3}{T}d(\frac{1}{\delta}-\frac{1}{G}) + \frac{4GL^2\rho^2}{\delta} \\& + \frac{4G\gamma_tL}{\delta}(\frac{L\rho^2}{\delta}+\frac{\sigma^2}{b\delta}) + \frac{4\gamma_tLG^3}{T}d(G^2-\delta^2)
    \\ &= \frac{2G  f(x_{0})-f^*}{\gamma_tT} + \frac{20GL^2\rho^2}{\delta} \\& + \frac{2G^3}{T}d(\frac{1}{\delta}-\frac{1}{G}) + \frac{4G\gamma_tL}{\delta}\frac{L\rho^2}{\delta} \\& + \frac{4G\gamma_tL}{\delta} \frac{\sigma^2}{b\delta} + \frac{4\gamma_tLG^3}{T}d(G^2-\delta^2)
\end{align*}
\end{proof}

\begin{theorem}
With probability $1 - \delta$ over the choice of training set $\mathcal{S}\sim \mathcal{D}$, we have the following inequality.
\begin{align*}
    &L_{\mathcal{D}}(w) \leq \max_{||\epsilon||_2\leq \rho} L_{\mathcal{S}}(w+\epsilon) \\+& \sqrt{\frac{k\log (1+\frac{||w||_2^2}{\rho^2} (1+\sqrt{\frac{\log n}{k}})^2)+ 4\log \frac{n}{\delta} + O(1)}{n-1} }
\end{align*}
The $n = |\mathcal{S}|$, k is the number of weight $w$, and assuming that $L_{\mathcal{D}}(w) \leq \mathbb{E}_{\epsilon_i \sim \mathcal{N}(0,\rho)}[L_{\mathcal{D}}(w+\epsilon)]$
\end{theorem}

\begin{proof}
The proof is mainly based on PAC-Bayesian Generalization Bound theorem. We start by  any prior $\mathcal{P}$ over parameters with probability $1 - \delta$ for any posterior distribution $\mathcal{Q}$, we have
\begin{align*}
    \mathbb{E}_{w\sim \mathcal{W}}[L_{\mathcal{Q}](w)}]& \leq \mathbb{E}_{w\sim \mathcal{W}}[L_{\mathcal{S}](w)}] \\+& \sqrt{\frac{KL(\mathcal{Q}||\mathcal{P}) + \log \frac{n}{\delta}}{2(n-1)}}
\end{align*}
Assume $\epsilon \sim \mathcal{N}(0,\sigma)$, the $||\epsilon||_2^2$ follows the Chi-square distribution. From the Lemma 1 in~\cite{laurent2000adaptive}, we have the following inequality for any positive $t$:
\begin{align*}
    P(||\epsilon||_2^2-k\sigma^2 \geq 2 \sigma^2 \sqrt{kt}+2t\sigma^2) \leq \exp (-t) 
\end{align*}
Combine the above function and subtract the same constant $C$ on both side, following the assumption in~\cite{foret2020sharpness}, we finish the proof.

\end{proof}

\input{tables/datasets}

\subsection{Details of Tasks and Datasets}
\label{appendix_data}
As mention in Section~\ref{sec_setup}, we conduct extensive experiments on parts of tasks from GLUE and SuperGLUE. In addition, two widely-used generation tasks are also used in this work. Here, we introduce the descriptions of the used tasks and datasets in detail. Firstly, we present the statistics of all datasets in Table~\ref{tab_data}. Then, each task is described as:

\textbf{CoLA.} Corpus of Linguistic Acceptability~\cite{warstadt2019neural} is a binary single-sentence classification task to determine whether a given sentence is linguistically ``acceptable''.

\textbf{MRPC.} Microsoft Research Paraphrase Corpus~\cite{dolan2005automatically} is a task to predict whether two sentences are semantically equivalent.

\textbf{STS-B.} Semantic Textual Similarity~\cite{cer2017semeval} is a task to predict how similar two sentences are on a 1-5 scale in terms of semantic meaning.

\textbf{RTE.} Recognizing Textual Entailment~\cite{giampiccolo2007third}, given a premise and a hypothesis, is a task to predict whether the premise entails the hypothesis. 

\textbf{QNLI.} Question Natural Language Inference is a binary classification task constructed from SQuAD~\cite{rajpurkar2016squad}, which aims to predict whether a context sentence contains the answer to a question sentence. 

\textbf{CB.} CommitmentBank~\cite{de2019commitmentbank} is a task that can be framed as three-class textual entailment on a corpus of 1,200 naturally occurring discourses.

\textbf{BoolQ.} Boolean Question~\cite{clark2019boolq} is a question answering task where each sample consists of a short passage and a yes/no question about the passage. 

\textbf{WiC.} Word-in-Context~\cite{pilehvar2019wic} is a word sense disambiguation task that aims to predict whether the word is used with the same sense in sentence pairs.

\textbf{WSC.} Winograd Schema Challenge~\cite{levesque2012winograd} is a co-reference resolution task which aims to determine the correct refer-rent of the pronoun from among the provided choices.

\textbf{XSUM.} The Extreme Summarization dataset~\cite{narayan2018don} is one of abstractive Summarization task that aims to convert the given document into a short and adequate summary in the same language.

\textbf{CoNLL2014.} CoNLL2014~\cite{ng2014conll} is a popular grammatical error correction task that aims to rewrite the input sentence with grammatical errors into the corresponding correct sentence, where the original and target sentences have the similar sentence lengths.

\input{tables/hyper-parameter2}
\subsection{Hyper-parameters of Fine-tuning}
\label{appendix_parameters}
In this paper, we fine-tune four different large-scale PLMs with our FSAM on the tasks of GLUE and SuperGLUE, including BERT-large ($\sim$340M)~\footnote{\url{https://huggingface.co/bert-large-cased}}, ELECTRA-large ($\sim$340M)~\footnote{\url{https://huggingface.co/google/electra-large-discriminator}}, ALBERT-xxlarge-v2 ($\sim$223M)~\footnote{\url{https://huggingface.co/albert-xxlarge-v2}} and RoBERTa-large ($\sim$355M)~\footnote{\url{https://huggingface.co/roberta-large}}. Additionally, the BART-large ($\sim$406M)~\footnote{\url{https://huggingface.co/facebook/bart-large}} is used for the generation tasks. The training epochs/steps, batch size, learning rate and warmup steps are listed in Table~\ref{tab_hyper2} and Tabel~\ref{tab_hyper}. Notably, the maximum sequence length of language understanding tasks is set as 128/256. For two generation tasks, we empirically set the minimum and maximum length of the XSUM dataset as 10 and 60, and closely follow~\citet{Chollampatt:18} to preprocess the data of CoNLL2014.

\subsection{Training Curves}
\label{appendix_curves}
In this sub-section, we visualize the training curves of Adam, SAM and FSAM in detail. Specifically, Figure~\ref{fig:training_acc} shows the evaluation metrics \textit{v.s.} training epochs. 
The metric curves prove that FSAM boosts the performance effectively during the training, which shows the effectiveness of FSAM.


\subsection{More Results}
In addition to the results in Table~\ref{tab_main} with the momentum as 0.9, we also conduct the experiments with the momentum as 0 (\textit{i.e.}, $\beta_1 = 0$ in Algorithm~\ref{alg:fsam}) to evaluate the influence of the adaptive learning rate. In practice, we evaluate the performance on several downstream tasks upon two base optimizers (Adam and AMSGrad) and two PLMs (RoBERTa-large and RoBERTa-base). Table~\ref{tab_more_result} shows the results. When there is no momentum term, FSAM also achieves better performance against the vanilla SAM and base optimizers. These results prove the universality of FSAM in various scenarios.

\input{figures/training_curves-acc}
\input{tables/more_results}

\input{tables/hyper-parameters}

%% file: tables/datasets.tex
\begin{table*}[h]
\caption{Data statistics of all used tasks in this paper.}
\label{tab_data}
\centering
\begin{tabular}{llcccc}
\toprule
\multicolumn{2}{c}{\textbf{Task}}                & \textbf{Description}                             & \textbf{\#Train} & \textbf{\#Dev}  & \textbf{\#Class} \\ \hline \hline
\multirow{5}{*}{GLUE}       & CoLA      & linguistic acceptability classification & 8.5K    & 1,042  & 2       \\
                            & MRPC      & paraphrase classification               & 3.7K    & 409    & 2       \\
                            & STS-B     & semantic textual similarity             & 5.7K    & 1,501  & -       \\
                            & RTE       & natural language inference              & 2.5K    & 278    & 2       \\
                            & QNLI      & natural language inference              & 104K    & 5,464  & 2       \\ \midrule
\multirow{4}{*}{SuperGLUE}  & BoolQ     & question answering                      & 9.4K    & 3,270  & 2       \\
                            & CB        & natural language inference              & 250     & 57     & 2       \\
                            & WiC       & word sense disambiguation               & 6K      & 638    & 2       \\
                            & WSC       & coreference resolution                  & 554     & 104    & 2       \\ \midrule
\multirow{2}{*}{Generation} & XSUM      & abstractive summarization               & 204K    & 11,332 & -       \\
                            & CoNLL2014 & grammatical error correction            & 1.3M    & 5,448  & -     \\
\bottomrule
\end{tabular}
\end{table*}

%% file: tables/hyper-parameter2.tex
\begin{table}[]
\caption{Hyper-parameters settings for BART model on the generation tasks.}
\label{tab_hyper2}
\centering
\begin{tabular}{lcc}
\toprule
\textbf{Setting}        & \textbf{XSUM}   & \textbf{CoNLL2014} \\
\hline \hline
Learning Rate  & 3e-5   & 2e-5      \\
Batch Size     & 56     & 800       \\
Training Steps & 15,000 & 15,000    \\
Warmup Steps   & 500    & 500       \\
GPUs           & 8      & 4   \\
\bottomrule
\end{tabular}
\end{table}

%% file: figures/training_curves-acc.tex
\begin{figure*}[ht]
	\centering
	\includegraphics[width=1\textwidth]{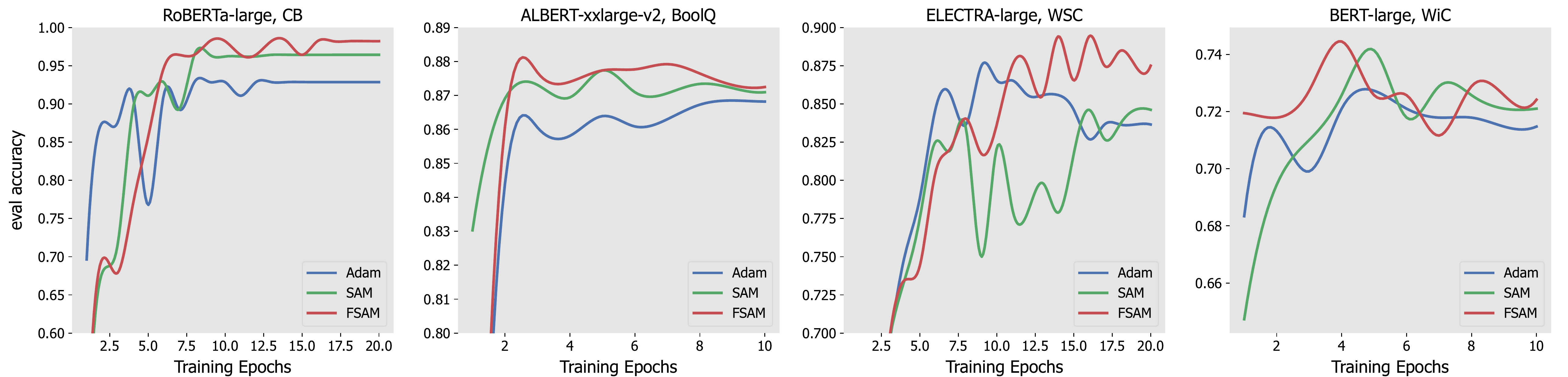} 
	\caption{Accuracy on the dev set \textit{v.s.} training epochs. We show some representative results on the four pretrained models. We can observe that our FSAM achieves the best performance among all tasks.
	}
	\label{fig:training_acc}
\end{figure*}

%% file: tables/more_results.tex
\begin{table*}[ht]
\caption{Results of RoBERTa-large and RoBERTa-base fine-tuned with different optimizers ($\beta_1 = 0$). We can observe that FSAM consistently outperforms the vanilla SAM and base optimizer among all settings.}
\label{tab_more_result}
\centering
\scalebox{0.78}{
\begin{tabular}{lcccccccccccccc}
\toprule
\multicolumn{1}{c}{\multirow{3}{*}{\textbf{Method}}} & \multicolumn{7}{c}{\textbf{RoBERTa-base, Adam}}                                                                                & \multicolumn{7}{c}{\textbf{RoBERTa-large, Adam}}                                                                               \\ \cmidrule(lr){2-8} \cmidrule(lr){9-15}
\multicolumn{1}{c}{}                        & CoLA          & \multicolumn{2}{c}{MRPC}      & \multicolumn{2}{c}{STS-B}     & RTE           & \multirow{2}{*}{AVG.} & CoLA          & \multicolumn{2}{c}{MRPC}      & \multicolumn{2}{c}{STS-B}     & RTE           & \multirow{2}{*}{AVG.} \\  \cline{2-7} \cline{9-14}
\multicolumn{1}{c}{}                        & \textit{Mcc.}          & \textit{Acc.}          & \textit{F1.}           & \textit{Pear.}       & \textit{Spea.}       & \textit{Acc.}          &                       & \textit{Mcc.}          & \textit{Acc.}          & \textit{F1.}           & \textit{Pear.}       & \textit{Spea.}       & \textit{Acc.}          &                       \\ \hline \hline
Adam                                        & 61.3          & 87.5          & 90.6          & 90.4          & 90.2          & 78.3          & \underline{83.05}                 & 67.0          & 89.0          & 91.9          & \textbf{91.9} & \textbf{91.6} & 88.4          & \underline{86.63}                 \\
Adam+SAM                                    & 61.6          & 88.2          & 91.6          & 90.5          & 90.2          & 79.4          & \underline{83.58}                 & 68.6          & \textbf{90.4} & 93.1          & 90.7          & 90.4          & 88.1          & \underline{86.88}                 \\
Adam+FSAM                                   & \textbf{63.2} & \textbf{88.5} & \textbf{92.0} & \textbf{90.6} & \textbf{90.3} & \textbf{82.7} & \underline{\textbf{84.55}}        & \textbf{69.0} & \textbf{90.4} & \textbf{93.2} & 91.6          & 91.3          & \textbf{88.8} & \underline{\textbf{87.38}}        \\ \midrule
\multicolumn{1}{c}{\multirow{3}{*}{\textbf{Method}}} & \multicolumn{7}{c}{\textbf{RoBERTa-base, AMSGrad}}                                                                                & \multicolumn{7}{c}{\textbf{RoBERTa-large, AMSGrad}}                                                                               \\ \cmidrule(lr){2-8} \cmidrule(lr){9-15}
\multicolumn{1}{c}{}                        & CoLA          & \multicolumn{2}{c}{MRPC}      & \multicolumn{2}{c}{STS-B}     & RTE           & \multirow{2}{*}{AVG.} & CoLA          & \multicolumn{2}{c}{MRPC}      & \multicolumn{2}{c}{STS-B}     & RTE           & \multirow{2}{*}{AVG.} \\  \cline{2-7} \cline{9-14}
\multicolumn{1}{c}{}                        & \textit{Mcc.}          & \textit{Acc.}          & \textit{F1.}           & \textit{Pear.}       & \textit{Spea.}       & \textit{Acc.}          &                       & \textit{Mcc.}          & \textit{Acc.}          & \textit{F1.}           & \textit{Pear.}       & \textit{Spea.}       & \textit{Acc.}          &                       \\  \hline \hline
AMSGrad                                     & 60.1          & 88.5          & 91.6          & 90.3          & 90.2          & 79.1          & \underline{83.30}                 & 63.8          & 89.7          & 92.4          & 90.0          & 90.4          & 87.4          & \underline{85.62}                 \\
AMSGrad+SAM                                 & 60.7          & \textbf{89.2} & \textbf{92.2} & 90.2          & 89.9          & \textbf{80.2} & \underline{83.73}                 & \textbf{68.5} & 90.0          & 92.8          & \textbf{91.6} & \textbf{91.2} & 87.0          & \underline{86.85}                 \\
AMSGrad+FSAM                                & \textbf{62.0} & 88.7          & 91.9          & \textbf{90.5} & \textbf{90.4} & 79.8          & \underline{\textbf{83.88}}        & \textbf{68.5} & \textbf{90.2} & \textbf{92.9} & \textbf{91.6} & 91.0          & \textbf{88.1} & \underline{\textbf{87.05}}      
\\ \bottomrule
\end{tabular}
}
\end{table*}

%% file: tables/hyper-parameters.tex
\begin{table*}[]
\caption{Hyper-parameters settings for different pretrained models on the language understanding tasks. We set the batch size to 16 for all settings. These settings are selected in best practice. Note that we apply these settings to fine-tune both large and small pretrained language models.}
\label{tab_hyper}
\centering
\begin{tabular}{ccccc}
\toprule
\textbf{Model}                    & \textbf{Dataset} & \textbf{Learning Rate} & \textbf{Training Epochs/Steps} & \textbf{Warmup Ratio/Steps} \\ \hline \hline
\multirow{8}{*}{\textbf{BERT}}    & CoLA    & 2e-5          & 3 epochs              & 10\%               \\
                         & MRPC    & 2e-5          & 3 epochs              & 10\%               \\
                         & STS-B   & 4e-5          & 3 epochs              & 10\%               \\
                         & RTE     & 4e-5          & 3 epochs              & 10\%               \\
                         & BoolQ   & 2e-5          & 10 epochs             & 10\%               \\
                         & CB      & 3e-5          & 20 epochs             & 10\%               \\
                         & WiC     & 2e-5          & 10 epochs             & 10\%               \\
                         & WSC     & 1e-5          & 20 epochs             & 10\%               \\ \midrule
\multirow{8}{*}{\textbf{ELECTRA}} & CoLA    & 1e-5          & 3 epochs              & 10\%               \\
                         & MRPC    & 3e-5          & 3 epochs              & 10\%               \\
                         & STS-B   & 3e-5          & 10 epochs             & 10\%               \\
                         & RTE     & 2e-5          & 10 epochs             & 10\%               \\
                         & BoolQ   & 2e-5          & 10 epochs             & 10\%               \\
                         & CB      & 3e-5          & 20 epochs             & 10\%               \\
                         & WiC     & 2e-5          & 10 epochs             & 10\%               \\
                         & WSC     & 1e-5          & 20 epochs             & 10\%               \\ \midrule
\multirow{8}{*}{\textbf{ALBERT}}  & CoLA    & 2e-5          & 3 epochs              & 10\%               \\
                         & MRPC    & 3e-5          & 3 epochs              & 10\%               \\
                         & STS-B   & 3e-5          & 3 epochs              & 10\%               \\
                         & RTE     & 5e-5          & 3 epochs              & 10\%               \\
                         & BoolQ   & 1e-5          & 10 epochs             & 10\%               \\
                         & CB      & 4e-5          & 20 epochs             & 10\%               \\
                         & WiC     & 3e-5          & 10 epochs             & 10\%               \\
                         & WSC     & 1e-5          & 20 epochs             & 10\%               \\ \midrule
\multirow{8}{*}{\textbf{RoBERTa}} & CoLA    & 1e-5          & 2668 steps            & 160 steps          \\
                         & MRPC    & 1e-5          & 1148 steps            & 68 steps           \\
                         & STS-B   & 2e-5          & 1799 steps            & 107 steps          \\
                         & RTE     & 2e-5          & 1018 steps            & 61 steps           \\
                         & BoolQ   & 1e-5          & 10 epochs             & 10\%               \\
                         & CB      & 2e-5          & 20 epochs             & 10\%               \\
                         & WiC     & 2e-5          & 10 epochs             & 10\%               \\
                         & WSC     & 1e-5          & 20 epochs             & 10\%         \\
\bottomrule
\end{tabular}
\end{table*}